\documentclass{article}

\usepackage{algorithmic, algorithm, xcolor, wrapfig, fancyvrb}

\usepackage[nonatbib, preprint]{neurips_2020_tweaked}
\usepackage[utf8]{inputenc}

\usepackage{amsmath, amsfonts, hyperref, url, booktabs, nicefrac, microtype, amssymb, accents, enumerate, amsthm}
\usepackage[T1]{fontenc}
\newcommand{\divergence}{\textnormal{div}\,}
\usepackage[pdftex]{graphicx}
\newcommand{\bbX}{\mathbb{X}}
\newcommand{\bbY}{\mathbb{Y}}
\newcommand{\bbZ}{\mathbb{Z}}
\newcommand{\RR}{\mathbb{R}}
\newcommand{\cX}{\mathcal{X}}
\newcommand{\cY}{\mathcal{Y}}
\newcommand{\cZ}{\mathcal{Z}}
\newcommand{\cL}{\mathcal{L}}
\newcommand{\dom}{\textnormal{dom}\,}
\newcommand{\nnnl}{\nonumber \\}
\newcommand{\prox}{\textnormal{Prox}}

\newcommand{\ee}{\mathbf{e}}
\newcommand{\citep}{\cite}
\newcommand{\citet}{\cite}

\newtheorem{theorem}{Theorem}
\newtheorem{lemma}{Lemma}

\DeclareMathOperator*{\argmax}{arg\,max}
\DeclareMathOperator*{\argmin}{arg\,min}

\usepackage{xcolor}

\VerbatimFootnotes 

\title{Increasing iterate averaging for solving saddle-point problems\\}

\author{
	Yuan Gao \\
Department of IEOR, Columbia University \\
New York, NY, 10027 \\
\texttt{gao.yuan@columbia.edu} \\
\And
Christian Kroer \\
Department of IEOR, Columbia University \\
New York, NY, 10027 \\
\texttt{christian.kroer@columbia.edu}
\And
Donald Goldrarb \\
Department of IEOR, Columbia University \\
New York, NY, 10027 \\
\texttt{goldfarb@columbia.edu}
}

\begin{document}
	\maketitle
	\begin{abstract} 
		Many problems in machine learning and game theory can be formulated as saddle-point problems, for which various first-order methods have been developed and proven efficient in practice. Under the general convex-concave assumption, most first-order methods only guarantee an ergodic convergence rate, that is, the uniform averages of the iterates converge at a $O(1/T)$ rate in terms of the saddle-point residual. However, numerically, the iterates themselves can often converge much faster than the uniform averages. This observation motivates increasing averaging schemes that put more weight on later iterates, in contrast to the usual uniform averaging. We show that such increasing averaging schemes, applied to various first-order methods, are able to preserve the $O(1/T)$ convergence rate with no additional assumptions or computational overhead. Extensive numerical experiments on zero-sum game solving, market equilibrium computation and image denoising demonstrate the effectiveness of the proposed schemes. In particular, the increasing averages consistently outperform the uniform averages in all test problems by orders of magnitude. When solving matrix and extensive-form games, increasing averages consistently outperform the last iterates as well. For matrix games, a first-order method equipped with increasing averaging outperforms the highly competitive CFR$^+$ algorithm.
	\end{abstract}
	\section{Introduction} \label{sec:intro} \vspace{-2px}
	Consider saddle point problems of the form
	\begin{align}
	\min_{x\in \bbX } \max_{y\in \bbY} \cL(x,y) \label{eq:spp-begin-general}
	\end{align}
	where $\cL$ is a general convex-concave function and $\bbX$, $\bbY$ are Euclidean spaces. For any $(x,y)\in \bbX\times \bbY$, denote its \textit{saddle-point residual} (SPR) as
	\begin{align*}
	\epsilon_{\rm sad}(x,y) = \max_{y'\in \bbY} \cL(x, y') - \min_{x'\in \bbX} \cL(x',y). 
	\end{align*}
	Many problems in machine learning \cite{juditsky2011first, chambolle2016ergodic}, imaging \cite{chambolle2011first} and game theory \cite{koller1996efficient, kroer2018faster} can be formulated as \eqref{eq:spp-begin-general}. First-order methods (FOMs) are naturally suitable for these problems and have been proven efficient in practice. For example, \citet{chambolle2011first} gives an algorithm for solving saddle-point problems involving bilinear and separable, nonsmooth terms and demonstrate its effectiveness in imaging applications. \citet{kroer2018solving} uses the Excessive Gap Technique (EGT)~\citep{nesterov2005excessive}, with a specific distance-generating function, to solve saddle-point formulation of zero-sum extensive-form games (EFG).
	Given the general convex-concave structure without strong convexity and smoothness assumptions, these algorithms only guarantee \textit{ergodic} convergence rate, that is, $\epsilon_{\rm sad}\left( \frac{1}{T}\sum_{t=1}^T x^t, \frac{1}{T} \sum_{t=1}^T y^t \right) = O(1/T)$. Meanwhile, numerically, $(x^t, y^t)$ (the ``last iterates'') often converge much more rapidly (see, e.g., \cite[\S 7.2.2]{chambolle2016ergodic}). This observation motivates new averaging schemes that put more weight on later iterates rather than uniformly across all of them. Let $(x^t, y^t)$, $t = 1, 2, \dots $ denote the iterates generated by a first-order method (more specifically, the iterates used in forming the convergent uniform averages; for certain algorithms, they are not necessarily denoted as $(x^t,y^t)$; see, e.g., Theorem \ref{thm:relax-pd-weighted}). Let $w_t$ be positive, nondecreasing weights. We consider averages of the form
	\begin{align}
	\bar{x}^T = \frac{1}{S_T} \sum_{t=1}^T w_t x^t,\ \  \bar{y}^T = \frac{1}{S_T} \sum_{t=1}^T w_t y^t, \ {\rm where}\ S_T = \sum_{t=1}^T w_t. \label{eq:ave-iter-def}
	\end{align}
	For example, $w_t = 1, t, t^2,$ and $ t^3$ result in \textit{uniform}, \textit{linear}, \textit{quadratic} and \textit{cubic} averages, respectively. We refer to such choices of positive, nondecreasing $w_t$ as \textit{increasing iterate averaging schemes} (IIAS) and the resulting $\bar{x}^T, \bar{y}^T$ as \textit{increasing averages}. In fact, in solving extensive-form games, the highly successful CFR$^+$ algorithm uses a form of linear averaging \cite{tammelin2015solving}. Similar averaging techniques have also been used in other scenarios, such as algorithms for solving large-scale sequential games that achieve superhuman performance in poker~\citep{bowling2015heads, moravvcik2017deepstack, brown2018superhuman}, and efficient large-scale GAN training \cite{yazici2018unusual}. On the theory side, \cite{golowich2020last} shows that, for unconstrained smooth saddle-point problems, the last iterates of Extragradient, a primal-dual FOM, converges slower than the averages in a strict sense. Through a unified analysis, \cite{davis2016convergence} shows $O(1/\sqrt{T})$ ergordic and last-iterate convergence rates of general splitting schemes, which apply to the special case of the general standard form \eqref{eq:sp-fgh*} with $f=0$. To the best of our knowledge, no theoretical justification has been given for the practical speedup of last iterates as compared to uniform averages, as well as the potentially harder problem of justifying the speedup from increasing iterate averaging.
	
	We show that for a wide class of FOMs, IIAS produces averages that converges at a rate $O(1/T)$ in terms of SPR. Algorithms compatible with IIAS include the (vanilla) primal-dual algorithm (PDA)~\cite{chambolle2011first}, its relaxed version (RPDA), inertial version (IPDA)~\citep{chambolle2016ergodic}, and linesearch version (PDAL)~\citep{malitsky2018first}, as well as Mirror Descent (MD) \cite{nemirovsky1983problem, beck2003mirror, ben2019lectures} and Mirror Prox (MP) \cite{nemirovski2004prox, ben2019lectures}. For most of the algorithms, in order to preserve the convergence of $(\bar{x}^T, \bar{y}^T)$, it suffices to choose $w_t = t^q$ for some weight exponent $q\geq 0$, completely independent of the problem instance and the algorithm. For algorithms with inertial terms or linesearch subroutine, in order to ensure the theoretical convergence rate, $w_t$ needs to satisfy additional inequalities, which makes $w_t$ depend on previous iterations. Still, simple formulas suffice (e.g., Theorem \ref{thm:inertial-pd-weighted} and \ref{thm:pdal}). Finally, we emphasize that for all first-order methods considered here, IIAS does \textit{not} alter the execution of the original algorithm. In other words, the performance boost is achieved without any extra computation or memory - we simply replace the uniform averages by increasingly weighted ones. Meanwhile, the averaging weights $w_t$, the sum of weights $\sum_{t=1}^T w_t$ and the averages $(\bar{x}^T, \bar{y}^T)$ can all be updated incrementally along the way.
	
	\textbf{Summary of contributions } First, we provide easily implementable IIAS for a variety of FOMs and establish their respective convergence properties. The high-level idea of the analysis can be summarized as follows. For each of the first-order methods, in the proof of the $O(1/T)$ rate of convergence, we identify the \textit{critical inequality} being summed across all time steps to derive the final rate. We then take a weighted sum instead where the weights are the increasing averaging weights $w_t$. Then, through telescoping the summation, we bound the right hand side by $O\left(w_T/S_T\right)$, which is $O(1/T)$ (with an extra constant $(q+1)$ compared to the original ones under uniform averaging) as long as $w_t$ grows \textit{polynomially}, that is, $w_t>0$, nondecreasing, and $\frac{w_{t+1}}{w_t} \leq \frac{(t+1)^q}{t^q}$ for all $t$, for some $q\geq 0$. Second, we perform extensive numerical experiments on various first-order methods and saddle-point point problems to demonstrate the consistent, strong performance gain of IIAS. Test problems include matrix games of different sizes and generative distributions, extensive-form games, Fisher market equilibirum computation and the TV-$\ell_1$ image denoising model. As the results demonstrate, increasing averages consistently outperform uniform averages in all setups by orders of magnitude. When solving matrix and extensive-form games, increasing averages consistently outperform the last iterates as well. For matrix games, PDA and RPDA equipped with IIAS also outperforms the highly competitive CFR$^+$ algorithm. For EFGs, RPDA under \textit{static, theoretically safe} hyperparameters equipped with quadratic averaging outperforms EGT with unsafe, sophisticated, adaptive stepsizing.
	
	\textbf{Organization } Section \ref{sec:pda-vanilla} presents the increasing averaging generalization of the primal-dual algorithm (PDA) of \cite{chambolle2016ergodic} and its analysis in detail. Section \ref{sec:pd-extensions} presents similar generalizations for the relaxed, inertial and linesearch variants of PDA. Section \ref{sec:mirror-algo} discusses similar generalizations for Mirror Prox and Mirror Descent. Section \ref{sec:experiments} presents numerical experiment setups and results. 
	
	\section{The primal-dual algorithm} \label{sec:pda-vanilla} \vspace{-2px}
	\textbf{Problem setup and notation } We follow the setup in \cite{chambolle2016ergodic}. Let $\bbX$ and $\mathbb{Y}$ be real Euclidean spaces with norms $\|\cdot \|_{\bbX}$ and $\|\cdot\|_{\mathbb{Y}}$, respectively. Denote the dual space of $\bbX$ as $\bbX^*$. Its corresponding dual norm, for any $x^* \in \bbX^*$, is defined as $\|x^*\|_{\bbX,*} = \sup_{\|x\| = 1} \langle x^*, x\rangle$. Define  $\mathbb{Y}^*$ and $\|y^*\|_{\mathbb{Y}, *}$ similarly. The subscripts on the norms are dropped when there is no ambiguity. Let $K: \bbX \rightarrow \bbY^*$ be a bounded linear operator. and $K^*: \bbY \rightarrow X^*$ be its adjoint operator. The (operator) norm of $K$ is defined as $\|K\| = \sup_{\|x\|\leq 1,\, \|y\|\leq 1} \langle Kx, y \rangle$.
	Let $\psi_{\bbX}$ and $\psi_{\bbY}$ be $1$-strongly convex (w.r.t. to their respective norms) smooth functions (known as distance-generating functions, DGF). Let $D_{\bbX}$ and $D_{\bbY}$ be their respective Bregman divergence functions, that is, for $\mathbb{V} = \bbX, \bbY$, $v,v'
	\in \mathbb{V}$,
	\begin{align}
	D_{\mathbb{V}}(v', v) := \psi_{\mathbb{V}}(v') - \psi_{\mathbb{V}}(v) - \langle \nabla \psi_{\mathbb{V}}(v), v' - v \rangle. \label{eq:breg-dist-from-dgf}
	\end{align}
	Let $f$ be a proper lower-semicontinuous (l$.$s$.$c$.$) convex function whose gradient $\nabla f$ is $L_f$-Lipschitz continuous w.r.t. $\|\cdot \|_\bbX$. Let $g, h$ be  proper l$.$s$.$c$.$ convex functions whose proximal maps $\prox_{\tau g}(x)= \argmin_u \left\{ \tau g(u) + D_{\bbX}(x, u) \right\}$ and $\prox_{\sigma h^*}(y) = \argmin_v \left\{ \sigma h^*(v) + D_{\bbY}(y, v) \right\}$, $\tau, \sigma>0$ can be easily computed. In addition, assume $\dom g \subseteq \dom \psi_{\bbX}$ and $\dom h^* \subseteq \dom \psi_{\bbY}$. Define the matrix $M_{\tau, \sigma} = \begin{bmatrix}
	\frac{1}{\tau} I & -K^* \\ -K & \frac{1}{\sigma} I 
	\end{bmatrix}$, which is positive definite (semidefinite) as long as $\tau, \sigma>0$ and $\tau \sigma L^2 < 1$ ($\leq 1$). With the above setup, consider the saddle-point problem (SPP)
	\begin{align}
	\min_{x\in \bbX}\max_{y\in \mathbb{Y}}\, \cL(x, y):= \langle K x, y \rangle +f(x)+g(x) - h^*(y). \label{eq:sp-fgh*}
	\end{align} 
	For $(\bar{x}, \bar{y}) \in \bbX\times \bbY$, $(\tilde{x}, \tilde{y}) \in \bbX\times \bbY$, the generic primal-dual iteration $(\hat{x}, \hat{y}) = PD_{\tau, \sigma}(\bar{x}, \bar{y}, \tilde{x}, \tilde{y})$ is
	\begin{align*}
	\hat{x} = \argmin_x \left\{
	\begin{matrix}
	f(\bar{x}) + \langle \nabla f (\bar{x}), x - \bar{x} \rangle + g(x)  \\ + \langle K x, \tilde{y}\rangle + \frac{1}{\tau} D_{\bbX}(x, \bar{x}) 
	\end{matrix} \right\},\ \ \hat{y} = \argmin_y \left\{ h^*(y) - \langle K \tilde{x}, y\rangle + \frac{1}{\sigma}D_{\bbY}(y, \bar{y}) \right\}.
	\end{align*}
	Here the update can be \textit{asymmetric}: if an algorithm uses $PD$ to generate iterates $(x^t, y^t)$, then $y^{t+1}$ may depend on $x^{t+1}$, which then depends on $(x^t, y^t)$. \citet{chambolle2016ergodic} proposes a primal-dual algorithm (PDA), which is listed here as Algorithm \ref{alg:pda}. Theorem 1 in their paper shows that the uniform averages converge at $O(1/T)$ in SPR. The key step in the proof is Lemma 1 there, which we restated below.
	\begin{lemma}\label{lem:pd-lemma}
		Assume $(\hat{x}, \hat{{y}}) = PD_{\tau, \sigma}(\bar{x}, \bar{y}, \tilde{x}, \tilde{y})$ for some $\tau, \sigma > 0$. Then, for any $(x,y) \in \bbX \times \bbY$, one has
		\begin{align*}
		& \cL(\hat{x}, y) - \cL(x, \hat{y})  \leq \frac{1}{\tau} \left( D_{\bbX}(x, \bar{x}) - D_{\bbX}(x, \hat{x}) - D_{\bbX}(\hat{x}, \bar{x})\right) +\frac{L_f}{2}\|\hat{x}-\bar{x}\|^2 \nnnl
		&\quad  + \frac{1}{\sigma}\left( D_{\bbY}(y, \bar{y}) - D_{\bbY}(y, \tilde{y}) - D_{\bbY}(\hat{y}, \bar{y}) \right) + \langle K(x-\hat{x}), \tilde{y} - \hat{y} \rangle  - \langle K(\tilde{x}-\hat{x}), y - \hat{y} \rangle.
		\end{align*}
	\end{lemma}
	
	\begin{algorithm}[tb]
		\caption{Nonlinear primal-dual algorithm (PDA)}
		\label{alg:pda}
		\begin{algorithmic}
			\STATE {\bfseries Input:} Initial iterate $(x^0, y^0)\in \bbX\times \bbY$, stepsizes $\tau, \sigma>0$, Bregman divergences $D_{\bbX}$ and $D_{\bbY}$.
			\STATE {\bfseries Iterations:} For $t = 0, 1, 2, \dots$, compute
			$(x^{t+1}, y^{t+1}) = PD_{\tau, \sigma}(x^t, y^t, 2x^{t+1} - x^t, y^t)$.
		\end{algorithmic}
	\end{algorithm}
	Based on Lemma \ref{lem:pd-lemma}, we obtain the following extension of Theorem 1 in \citet{chambolle2016ergodic} that incorporates IIAS.
	\begin{theorem} \label{thm:pd-vanilla-weighted}
		For $t = 0, 1, 2, \dots$, let $w_t = t^q$ for some $q\geq 0$ and $(x^t, y^t)$, $t = 1,2, \dots$ generated by \textnormal{PDA}, where stepsizes $\tau, \sigma$ are chosen such that, for all $x, x' \in \dom g$ and $y, y' \in \dom h^*$, it holds that
		\begin{align}
		\left(\frac{1}{\tau} - L_f\right) D_{\bbX}(x, x') + \frac{1}{\sigma} D_{\bbY}(y, y') 
		- \langle K(x - x'), y - y'\rangle \geq 0. \label{eq:pd-stepsize-cond}
		\end{align}
		Let $\bar{x}^T$, $\bar{y}^T$ be as in \eqref{eq:ave-iter-def}. 
		Denote \[A(x, y, x', y') = \frac{1}{\tau} D_{\bbX}(x, x') + \frac{1}{\sigma}D_{\bbY}(y, y') - \langle K(x-x'), y-y'\rangle.\] 
		Let $\Omega = \sup_{x,x'\in \dom g,\, y, y' \in \dom h^*} A(x, y, x', y')$. 
		Then, for any $T\geq 1$ and $(x,y) \in \bbX \times \bbY$, one has \[\cL(\bar{x}^T, y) - \cL(x, \bar{y}^T) \leq \frac{(q + 1)\Omega}{T}.\] 
	\end{theorem}
	\textit{Proof.}
		As in the proof of Theorem 1 in \cite{chambolle2016ergodic}, Lemma \ref{lem:pd-lemma} and PDA imply the following \textit{critical inequality}:
		\begin{align}
		& \cL(x^{t+1}, y) - \cL(x, y^{t+1}) \nnnl 
		& \leq A(x, y, x^t, y^t) - A(x, y, x^{t+1}, y^{t+1})  - \left( A(x^{t+1}, y^{t+1}, x^t, y^t) - \frac{L_f}{2} \|x^{t+1}-x^t \|^2 \right) \nnnl
		& \leq A(x, y, x^t, y^t) - A(x, y, x^{t+1}, y^{t+1})
		\label{eq:pd-vanilla-proof-critical-ineq}
		\end{align}
		where the last inequality is due to \eqref{eq:pd-stepsize-cond}.  Multiplying \eqref{eq:pd-vanilla-proof-critical-ineq} by $w_{t+1}$ and summing up over $t=0, 1, \dots, T-1$ yield
		\begin{align*}
			\sum_{t=1}^{T} w_t \left( \cL(x^t, y) - \cL(x, y^t) \right) & \leq \sum_{t=1}^{T} w_t \left( A(x,y,x^{t-1}, y^{t-1}) - A(x,y,x^t, y^t) \right) \\
			& = \sum_{t=1}^T (w_t - w_{t-1})  A(x,y,x^{t-1}, y^{t-1}) \leq \Omega w_T.
		\end{align*}
		The convex-concave structure of $\cL$ implies 
		\[\cL(\bar{x}^T, y) - \cL(x, \bar{y}^T) \leq \frac{1}{S_T}\sum_{t=1}^{T} w_t \left( \cL(x^t, y) - \cL(x, y^t) \right).\]
		Furthermore, \[S_T \geq \int_0^T x^q\, dx = \frac{T^{q+1}}{q+1}.\] Combining the above inequalities yields the claim.\qed 
		
	The key proof idea is to take the weighted sum of the critical inequalities at all $t$ and bound the right hand side via telescoping summation. In fact, this technique recurs in subsequent analysis of IIAS for other algorithms. Here, the bound degrades by a constant factor as $q$ increases, but numerical experiments show that a nonzero (often small) value of $q$ always yields significant speedup. 
	In addition, we remark that our result is no more restrictive than \cite{chambolle2016ergodic} in terms of the domain boundedness assumption (see Appendix \ref{app:domain-bounded} for more details). In fact, $\Omega$ often takes on a small, finite value in many realistic scenarios. For example, for a two-person zero-sum game, $g$ and $h^*$ are indicator functions of the strategy spaces $X, Y$ (which are simplexes for a matrix game), which are bounded polytopes with small diameters. 
	The linear map $K$ corresponds to the payoffs, which can be normalized to $\|K\| = 1$ w.l.o.g.
	Finally, note that Theorem \ref{thm:pd-vanilla-weighted} and subsequent theorems present \textit{point-wise} inequalities, that is, a uniform bound on $\cL(\bar{x}^T, y) - \cL(x, \bar{y}^T)$ independent of $(x,y)$. A bound on the saddle-point residual $\epsilon_{\rm sad}(\bar{x}^T, \bar{y}^T)$ can be easily obtained by taking $\min_{x\in \bbX} \max_{y\in \bbY}$ on both sides.\vspace{-5px}
	
	\section{Extensions of PDA}\label{sec:pd-extensions}\vspace{-2px}
	Similar IIAS can be applied to the relaxed and inertial versions of PDA, as described in \cite{chambolle2016ergodic}, as well as a nontrivial extension with linesearch (PDAL) \cite{malitsky2018first}. We use the same notation and setup as those in Section \ref{sec:pda-vanilla}. In addition, assume $\|\cdot \|_{\bbX}$, $\|\cdot\|_{\bbY}$ are Euclidean $2$-norms and $D_{\bbX}(x,x') = \frac{1}{2}\|x-x'\|_2^2$ and $D_{\bbY}(y,y') = \frac{1}{2}\|y - y'\|_2^2$.
	
	\textbf{Relaxed primal dual algorithm } The relaxed primal-dual algorithm (RPDA) in \cite{chambolle2016ergodic} is listed here as Algorithm \ref{alg:relax-pd}. We show a similar convergence regarding its IIAS. 
	The proof is similar to that of Theorem \ref{thm:pd-vanilla-weighted} and can be found in Appendix \ref{app:proof-relax}.
	\begin{theorem}
		Let $\tau, \sigma >0$ and $\rho \in (0,2)$ satisfy\footnote{Theorem 2 in \cite{chambolle2016ergodic} requires strict inequality instead. In fact, the nonstrict one suffices. The strict inequality is required for sequence convergence of the last iterates \cite[Remark 3]{chambolle2016ergodic}. The same is true for Theorem 3. However, there we assume strict inequality as in \cite[Theorem 3]{chambolle2016ergodic}, since it ensures that the weights grow polynomially eventually (so that IIAS does not reduce to uniform averaging).} $\left(\frac{1}{\tau} - \frac{L_f}{2-\rho}\right) \frac{1}{\sigma} \geq \|K\|_2^2$.
		Let $0\leq \rho_t \leq \rho_{t+1} \leq \rho$, $t=0, 1, 2 \dots$ Let $(\xi^t, \eta^t)$, $t = 1, 2, \dots$ be generated by \textnormal{RPDA}. Let $\bar{x}^T = \frac{1}{S_T}\sum_{t=1}^T w_t \xi^t$ and $\bar{y}^T = \frac{1}{S_T}\sum_{t=1}^T w_t \eta^t$. Then, for any $z = (x,y) \in \bbX\times \bbY$, one has
		\[\cL(\bar{x}^T, y) - \cL(x, \bar{y}^T) \leq \frac{(q+1)\Omega}{\rho_0 T},\] 
		where $\Omega = \sup_{z, z'\in \bbX \times \bbY} \frac{1}{2}\|z - z'\|_{M_{\tau, \sigma}}$. 
		\label{thm:relax-pd-weighted}
	\end{theorem}
	\begin{algorithm}[tb]
		\caption{Relaxed primal-dual algorithm (RPDA)}
		\label{alg:relax-pd}
		\begin{algorithmic}
			\STATE {\bfseries Input:} Initial iterates $z^0 = (x^0, y^0)\in \bbX\times \bbY$, stepsizes $\tau, \sigma>0$, relaxation parameters $\rho_t$.
			\STATE{\bfseries Set:} $D_{\bbX}(x,x') = \frac{1}{2}\|x-x'\|_2^2$ and $D_{\bbY}(y, y') = \frac{1}{2}\|y-y'\|_2^2$. \\
			Denote $z^t = (x^t, y^t)$ and $\zeta^t = (\xi^t, \eta^t)$.
			\STATE{\bfseries Iterations:} For $t=0, 1, 2, \dots$, compute
			\begin{align*}
			(\xi^{t+1}, \eta^{t+1}) = PD_{\tau, \sigma}(x^t, y^t, 2\xi^{t+1} - x^t, y^t), \ \
			z^{t+1} = (1-\rho_n) z^t + \rho_n \zeta^{t+1}.
			\end{align*} 
		\end{algorithmic}
	\end{algorithm}

	\textbf{Inertial primal-dual algorithm } Another useful variant of PDA is the inertial primal-dual algorithm (IPDA) \cite{chambolle2016ergodic}, listed here as Algorithm \ref{alg:intertial-pd}. In contrast to PDA and RPDA, for IPDA, in order to preserve the rate of convergence, $w_t$ needs to satisfy a few additional inequalities, which arise from analyzing the telescoping sum bound. They do not affect the asymptotic rate of convergence. The choice of $w_t$ and the convergence guarantee are summarized below. The proof is based on the same principles but is more involved; it can be found in Appendix \ref{app:proof-inertial}. 	Note that the strict inequality involving $\tau, \sigma, \|K\|$ is needed for $b_t\geq b^*>1$, which ensures eventual unbounded, polynomial growth of $w_t$.
	
	\begin{algorithm}[tb]
		\caption{Inertial primal-dual algorithm (IPDA)}
		\label{alg:intertial-pd}
		\begin{algorithmic}
			\STATE {\bfseries Input:} $(x^{-1}, y^{-1}) = (x^0, y^0)\in \bbX\times \bbY$, $\tau, \sigma>0$, inertial parameters $\alpha_t$.
			\STATE{\bfseries Set:} $D_{\bbX}(x,x') = \frac{1}{2}\|x-x'\|_2^2$ and $D_{\bbY}(y, y') = \frac{1}{2}\|y-y'\|_2^2$. \\
			Denote $z^t = (x^t, y^t)$ and $\zeta^t = (\xi^t, \eta^t)$.
			\STATE{\bfseries Iterations:} For $t = 0, 1, 2, \dots$, compute\newline \[\zeta^t = z^t + \alpha_t (z^t - z^{t-1}),\ \ z^{t+1} = PD_{\tau,\sigma}(\xi^t, \eta^t, 2 x^{t+1} - \xi^t, \eta^t).\]
		\end{algorithmic}
	\end{algorithm}
	
	\begin{theorem}
		Let $\alpha >0$ be such that
		$ \left(\frac{1}{\tau} - \frac{(1 + \alpha)^2}{1 - 3\alpha} L_f \right)\frac{1}{\sigma} > \|K\|^2_2$ and $0\leq \alpha_t \leq \alpha_{t+1} \leq \alpha < \frac{1}{3}$ for all $t$. Let $(x^t, y^t)$ be generated by \textnormal{IPDA}.  Let $w_t$ be chosen as follows: 
		\[w_0 = 0,\ w_1 = 1,\ w_{t+1} = w_t \cdot \min\left\{b_t, \frac{(t+1)^q}{t^q} \right\}, \ t \geq 2,\] 
		where \[b_t = \min \left\{\frac{1-\alpha_{t-1}}{\alpha_t}, \frac{r(1-\alpha_{t-1}) - (1+\alpha_{t-1})}{\alpha_t (1 + 2r + \alpha_t)}\right\}, \ r = \frac{\frac{1}{\tau} - \sigma\|K\|^2}{L_f}.\]
		Then, $w_t\leq w_{t+1}$ and $b_t \geq b^*$ for some $b^* > 1$ for all $t$. Furthermore, let $S_T$, $\bar{x}^T$ and $\bar{y}^T$ be as in \eqref{eq:ave-iter-def}. For any $(x, y)\in \bbX \times \bbY$, one has
		\[\cL(x, \bar{y}^T) - \cL(\bar{x}^T, y) \leq \frac{(1-\alpha_0)w_1 A_0 + \Omega \left[ (1-\alpha_{T-1})w_T + \alpha_1 w_2  - w_1 \right]}{S_T} \leq \frac{(q+1)(2-\alpha_0) \Omega }{T},\] 
		where $A_0 := \| z - z^0\|^2_{M_{\tau, \sigma}}$ and $\Omega$ is as in Theorem \ref{thm:relax-pd-weighted}.
		\label{thm:inertial-pd-weighted}
	\end{theorem}

	\textbf{Primal-dual algorithm with linesearch } Recently, \cite{malitsky2018first} proposed a primal-dual algorithm with linesearch (PDAL), which is listed as Algorithm \ref{alg:pdal} here. To align the primal and dual iterates in the increasing averaging version of PDAL, $y^t$ here correspond to $y^{t+1}$ in the original paper, $t = 0, 1, 2, \dots$ Assume $\|\cdot\|_{\bbX}$, $\|\cdot\|_{\bbY}$, $D_{\bbX}$ and $D_\bbY$ are all Euclidean as in Theorem \ref{thm:relax-pd-weighted}. Furthermore, assume $f = 0$ in \eqref{eq:sp-fgh*}. Theorem 3.5 in \cite{malitsky2018first} establishes the ergodic convergence rate of PDAL. We show the following theorem on IIAS for PDAL. Note that the averaging involves not only the weights $w_t$ but also the stepsizes $\tau_t$. Similar is true for the subsequent Mirror-type algorithms. The proof can be found in Appendix \ref{app:proof-pdal}. Here, denote $\Omega_\bbX = \sup_{x, x'\in \dom g}\frac{1}{2}\|x-x'\|^2$, $\Omega_\bbY = \sup_{y, y'\in \dom h^*} \frac{1}{2}\|y - y'\|^2$. 
	\begin{algorithm}[tb]
		\caption{PDAL: Primal-dual algorithm with linesearch}
		\label{alg:pdal}
		\begin{algorithmic}
			\STATE {\bfseries Input:} initial iterates $(x^0, y^0) \in \bbX\times \bbY$, initial stepsize $\tau_0>0$, backtracking discount factor $\mu$, backtracking break tolerance $\delta$, primal-dual ratio $\beta>0$. 
			\STATE{\bfseries Set:} stepsize growth factor $\theta_0 = 1$, $D_{\bbX}(x,x') = \frac{1}{2}\|x-x'\|_2^2$ and $D_{\bbY} = \frac{1}{2}\|x-x'\|_2^2$.
			\STATE{\bfseries Iterations:} For $t = 0, 1, 2, \dots$, compute
			$x^{t+1} = \prox_{\tau_t g }(x^t - \tau_t K^* y^t)$.
			\begin{enumerate}[(a)]
				\item \label{item:pdal-first-step-ls} Choose $\tau_{t+1} \in [\tau_t, \tau_t\sqrt{1+\theta_t}]$ and perform linesearch: compute $\theta_{t+1} = \frac{\tau_{t+1}}{\tau_t}$, $\tilde{x}^{t+1} = x^{t+1} + \theta_{t+1} (x^{t+1} - x^t)$ and 
				$y^{t+1} = \prox_{\beta \tau_{t+1} h^*}( y^t + \beta \tau_{t+1} K \tilde{x}^{t+1})$.
				\item Break if
				$\sqrt{\beta}\tau_{t+1} \|K^* y^{t+1} - K^* y^t\| \leq \delta \|y^{t+1} - y^t\|$.
				Otherwise, $\tau_{t+1} \leftarrow \tau_{t+1} \mu$, go to \eqref{item:pdal-first-step-ls}.
			\end{enumerate}
		\end{algorithmic}
	\end{algorithm}
	
	\begin{theorem} \label{thm:pdal}
		Let $(x^t, y^t)$ be generated by \textnormal{PDAL} for the saddle point problem \eqref{eq:sp-fgh*} with $f = 0$. Let $w_t$ be as follows: \[w_0 = 0,\ w_1 = 1,\ w_{t+1} = w_t \cdot \min\left\{ \frac{1+\theta_t}{\theta_{t+1}}, \frac{(t+1)^q}{t^q} \right\}, \ t \geq 1.\]
		Let \[S_T = \sum_{t=1}^T w_t \tau_t, \ \bar{x}^T = \frac{w_1 \theta_1 \tau_1 x_0 + \sum_{t=1}^N w_t \tau_t \tilde{x}^t }{w_t \tau_1 \theta_1 + S_T }, \ \bar{y}^T = \frac{\sum_{t=1}^T w_t\tau_t y^t }{S_T}.\] 
		For any $(x, y)\in \bbX \times \bbY$, it holds that, for all $T \geq 2$,
		\[\cL(\bar{x}^T, y) - \cL(x, \bar{y}^T) \leq \frac{w_T \left(\Omega_\bbX + \frac{1}{\beta} \Omega_\bbY\right) + w_1 \tau_1 \theta_1 P_0 }{S_T} \leq \frac{(q+1) \left(\Omega_\bbX + \frac{1}{\beta} \Omega_\bbY + \tau_1 \theta_1 P_0\right)}{T},\] where $P_0 = g(x^0) - g(x) + \langle K(x^0 - x), y \rangle$.
		\vspace{-8px}
	\end{theorem}
	\section{Mirror-type algorithms}\label{sec:mirror-algo}\vspace{-2px}
	A similar extension can also be applied to another class of algorithms for saddle-point problems, namely, Mirror Descent (MD) \cite{nemirovsky1983problem, beck2003mirror} and Mirror Prox (MP) \cite{nemirovski2004prox}. 
	These algorithms require a setup different from that of Section \ref{sec:pda-vanilla}.
	Specifically, assume that $\bbX$ and $\bbY$ are Euclidean spaces with norms $\|\cdot \|_\bbX$ and $\|\cdot\|_\bbY$. Let $\bbZ  = \bbX \times \bbY$ and $\|\cdot \|$ be be any norm (with dual norm $\|\cdot \|_*$) on $\bbZ$. Let $\cX \subseteq \bbX$ and $\cY \subseteq \bbY$ be closed convex sets. Let \[\phi: \cZ = \cX \times \cY \rightarrow \RR\] 
	be a convex-concave cost function, that is, $\phi(\cdot, y)$ is convex for any $y\in \cY$ and $\phi(x, \cdot)$ is concave for any $x \in \cX$. The saddle-point problem of interest is $\min_{x \in \cX} \max_{y \in \cY}\, \phi(x,y)$.
	Let $\psi_\bbZ: \bbZ \rightarrow \RR$ be a $1$-strongly convex (w.r.t. $\|\cdot \|_\bbZ$) smooth function, that is, a DGF on $\bbZ$. For example, given DGF $\psi_{\bbX}(x)$ and $\psi_\bbY$ for $\bbX$ and $\bbY$, the DGF $\psi_\bbZ(z) := \psi_{\bbX}(x) + \psi_\bbY(y)$, $z = (x,y)$, is $1$-strongly convex w.r.t. the norm $\|z\| := \sqrt{\|x\|_\bbX^2 + \|y\|_\bbY^2}$. Let the Bregman divergence function $D_\bbZ$ be defined as \eqref{eq:breg-dist-from-dgf} with $\mathbb{V} = \bbZ$. Let $\Omega = \sup_{z \in \cZ} D_{\bbZ}(z)$.\footnote{$\Omega$ here in fact corresponds to the constant $\Theta$ in \cite{ben2019lectures}, defined on page 356. There, instead, $\Omega := \sup_{z, z'\in \cZ} \psi_\bbZ(z,z')$ is the upper bound on the DGF.}
	The ``gradient vector field'' associated with $\phi$ is $F(z) = \left(\frac{\partial}{\partial_x}\phi(z), -\frac{\partial}{\partial_y}\phi(z)\right)$. We assume $F$ is bounded and $L$-Lipschitz continuous on $Z$, that is, $M_F = \sup_{z\in \cZ} \|F(z)\|_* < \infty$ and $\|F(x) - F(z')\|_* \leq L \|z - z'\|$ for any $z,z'\in \cZ$.
	For $z, \xi \in \bbZ$, define the (constrained) proximal mapping (of a linear function) as \[\prox^\cZ_{\langle \xi, \cdot \rangle}(z) := \argmin_{w\in \cZ} \left\{\langle \xi, w \rangle + D_\bbZ(w, z) \right\}.\] 
	The algorithms are listed together here as Algorithm \ref{alg:md-mp}. Similarly, we show that IIAS applied to MD and MP preserves their respective convergence rates, where the averaging weights involve the stepsizes $\tau_t$. During the course of this research, we found that IIAS for MP has already been analyzed in the growing lecture notes of \citet{ben2019lectures}. Specifically, Theorem 5.6.2 in \cite{ben2019lectures} presents the weighted averaging version of MP, though no explicit formulas for weights are given. 	The proof can be found in Appendix \ref{app:proof-md}. In addition,
	\cite[Theorem 5.3.5]{ben2019lectures} show that a similar bound holds for MD with increasing averaging.
	
	\begin{algorithm}[tb]
		\caption{Mirror Descent (MD) and Mirror Prox (MP)}
		\label{alg:md-mp}
		\begin{algorithmic}
			\STATE {\bfseries Input:} initial iterate $z^0 = (x^0, y^0)\in \bbZ$, stepsizes $\tau_t$.
			\STATE{\bfseries Iterations:} For $t = 0, 1, 2, \dots$, compute
			\begin{align*}
			& \textnormal{MD:}\  z^{t+1}  = \prox^\cZ_{\langle \tau_t F(z^t), \cdot \rangle}(z^t), \ \  \textnormal{MP:} \ 
			\tilde{z}^t  = \prox^\cZ_{\langle \tau_t F(z^t), \cdot \rangle}(z^t), \
			z^{t+1} = \prox^\cZ_{\langle \tau_t F(\tilde{z}^t), \cdot\rangle} (z^t)
			\end{align*}
		\end{algorithmic}
	\end{algorithm}
	\begin{theorem}
		Let $\tilde{z}^t = (\tilde{x}^t, \tilde{y}^t)$ be generated by \textnormal{MP} and
		$\delta_t := \tau_t \langle F(\tilde{z}^t), \tilde{z}_t - z_{t+1}\rangle - D_{\bbZ}(z_{t+1}, z_t)$. Let $w_t\geq 0$ be nondecreasing. Define
		\[S_T = \sum_{t=1}^T w_t \tau_t,\ \bar{z}^T = (\bar{x}^T, \bar{y}^T) = \frac{\sum_{t=1}^T w_t \tau_t \tilde{z}^t}{S_T}.\]  
		Then, for any $T\geq 1$ and any $(x,y)\in \bbX\times \bbY$, it holds that 
		\[\phi(\bar{x}^T, y) - \phi(x, \bar{y}^T) \leq \frac{w_T \Omega + \sum_{t=1}^T w_t \delta_t }{S_T}.\]  
		In particular, constant stepsizes $\tau_t = \frac{1}{L}$ and $w_t = t^q$, $q\geq 0$ ensure $\delta_t\leq 0$ and
		\[\phi(\bar{x}^T, y) - \phi(x, \bar{y}^T) \leq \frac{(q+1)L \Omega}{T}. \]
		\label{thm:mp}
	\end{theorem}
	\section{Numerical experiments}\label{sec:experiments} \vspace{-2px}
	We demonstrate the numerical speedup of increasing averaging schemes for matrix games, extensive-form games, Fisher-market-competitive-equilibrium computation, and image-denoising problems. For matrix games and EFG, our best FOMs with IIAS are compared against state-of-the-art methods for equilibrium computation.

	
	\textbf{Matrix games }
	Here, we describe the setups at a high level with further details deferred to Appendix \ref{app:details-game-exper}. A matrix game can be formulated as a bilinear SPP $\min_{x\in \Delta^{n_1}}\max_{y \in \Delta^{n_2}} \langle x, Ay\rangle$, where $\Delta^{k} := \{ x\in \RR^k \mid x^\top \ee = 1,\, x\geq 0 \}$. To use a PDA-type algorithm to solve a matrix game, let $g$ and $h^*$ be the indicator functions of $\Delta^{n_1}$ and $\Delta^{n_2}$, respectively, and $K = A^\top$. For MP, we take $\cX = \Delta^{n_1}$ and $\cY = \Delta^{n_2}$ and $\phi(x,y) = \langle x, Ay\rangle$. We also try a linesearch variant of MP, referred to as MPL, that performs linesearch as described in \cite[pp. 443]{ben2019lectures}. For all algorithms, we use their ``default'' stepsizes and Euclidean DGF. We generate matrix games of different dimensions with i.i.d. entries and solve them using all six algorithms. For each matrix game and each algorithm, we perform $T = 2000$ iterations, take increasing averages of the iterates and compute their saddle-point residuals. Residuals are normalized to make then commensurable in magnitude. The above is repeated $50$ times. We plot the averages and standard errors (which have small magnitudes and are nearly invisible) of the normalized residuals along sample paths of each setup-algorithm-averaging combination. Figure \ref{fig:fom-ave-rand-mat}displays the plots. As can be seen, IIAS leads to significant performance improvement for all algorithms across all experiment setups, both against uniform averages as expected, but, perhaps surprisingly, also against the last iterates.
	Next, we compare PDA, RPDA, and CFR$^+$ on the same set of random matrix games as well as on a $2\times 2$ matrix game with payoff matrix $A = \begin{bmatrix} 5 & -1 \\ 0 & 1 \end{bmatrix}$ \cite{farina2019optimistic}. Figure \ref{fig:fom-vs-cfr+} displays the results for these experiments: the upper plot is for the random matrix game experiments and displays averaged, normalized residuals and standard deviations of $3$ setups, similar to Figure \ref{fig:fom-ave-rand-mat}; the lower plot is for the $2\times 2$ matrix game and displays the (unnormalized) residual values. Clearly, PDA and RPDA outperform CFR$^+$ in both settings. Moreover, for the $2\times 2$ matrix game, the last iterates converge rapidly, suggesting the use of a large weight exponent $q$. As the lower subplot in Figure \ref{fig:fom-vs-cfr+} displays, using $q=10$ can even outperform the last iterates. 
	We stress that we test $q=10$ mostly for experimental curiosity. In general, using $q=1, 2$ yields significant speedup. Nevertheless, large $q$ does not lead to numerical issues in any experiment. See Appendix \ref{app:detail-large-q} for more details.
	
	\textbf{Extensive-form games } An EFG can be written as a bilinear saddle-point problem (BLSP) \[\min_{x\in \cX}\max_{y\in \cY}\, \langle x, Ay\rangle \label{blsp-game}\] 
	where $\cX\subseteq \RR^{n_1}$ and $\cY\subseteq \RR^{n_2}$ are polytopes encoding players' strategy spaces, known as \textit{treeplexes} \cite{hoda2010smoothing}.  We perform IIAS with uniform, linear, quadratic and cubic averaging for all first-order methods on two classic EFG benchmark instances \verb|Kuhn| and \verb|Leduc| poker.  The choices of the algorithms' hyperparameters are completely analogous to those in solving matrix games and is described in Appendix \ref{app:details-game-exper}. As Figure \ref{fig:fom-with-ias-on-efg} shows, increasing averages outperform uniform averages for all algorithms in both games. For all algorithms except MP, increasing averages also outperform the last iterates. For both games, we also compare RPDA with quadratic and $q=10$ averaging, CFR$^+$ \cite{tammelin2015solving} and EGT with the dilated entropy DGF \cite{kroer2018faster}. We plot the saddle-point residuals against number of \textit{gradient computations} $x \mapsto A^\top x$ and $y \mapsto A y$, since EGT uses linesearch and requires a varying number of gradient computations in each iteration. Here, since we are interested in the regime where gradient computations dominate the overall computation cost, we assume computing the proximal mappings under different DGFs takes much less time in comparison. As Figure \ref{fig:rpda-egt-cfr+-on-efg} shows, RPDA with quadratic and $q=10$ averaging significantly outperforms CFR$^+$ and EGT on \verb|Kuhn| but are not as fast as CFR$^+$ on \verb|Leduc|. RPDA with quadratic averaging and using theoretically safe, highly conservative stepsizes outperforms EGT, which employs sophisticated, unsafe, adaptive stepsizing. Furthermore, Appendix \ref{app:precond-rpda-cfr} presents additional experiments using preconditioned PDA and RPDA \cite{pock2011diagonal} with IIAS, bringing further speedup.
	\begin{figure}
		\includegraphics[width=0.5\columnwidth]{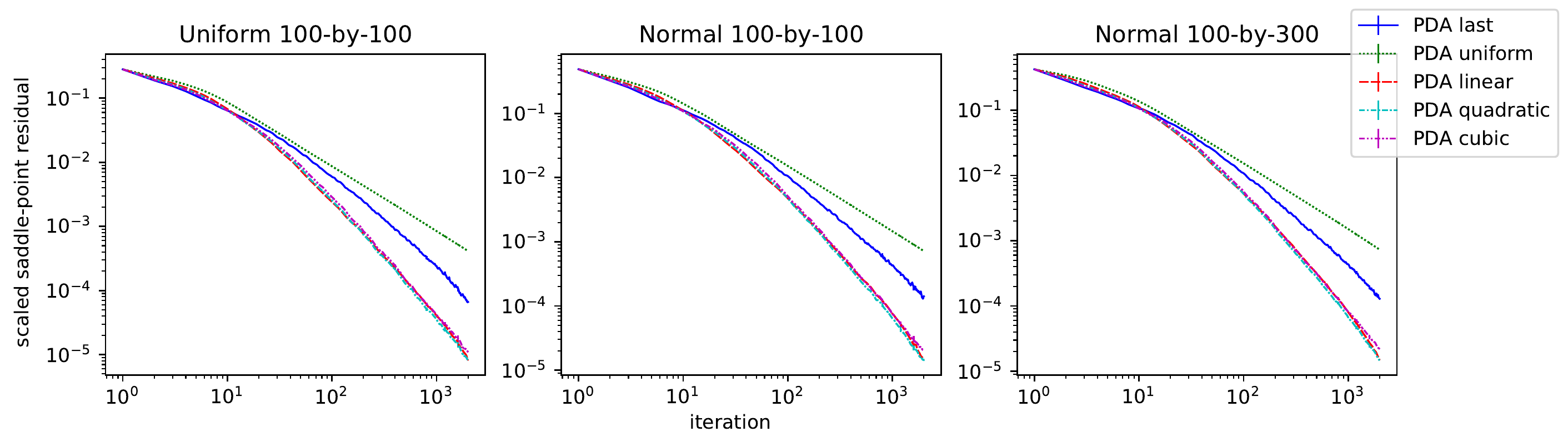} \includegraphics[width=0.5\columnwidth]{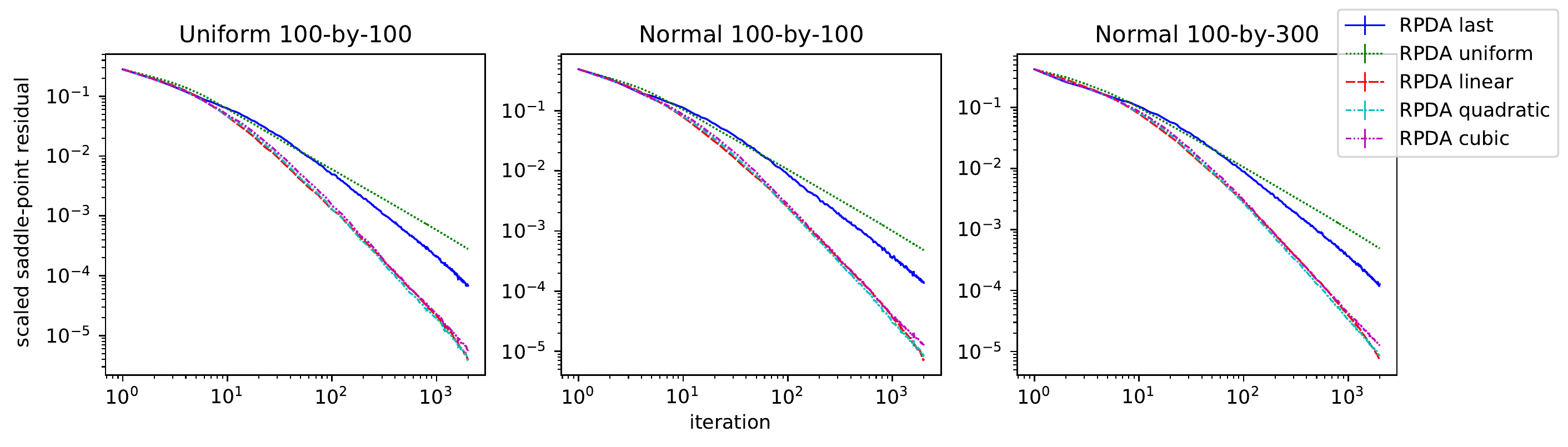}
		\includegraphics[width=0.5\columnwidth]{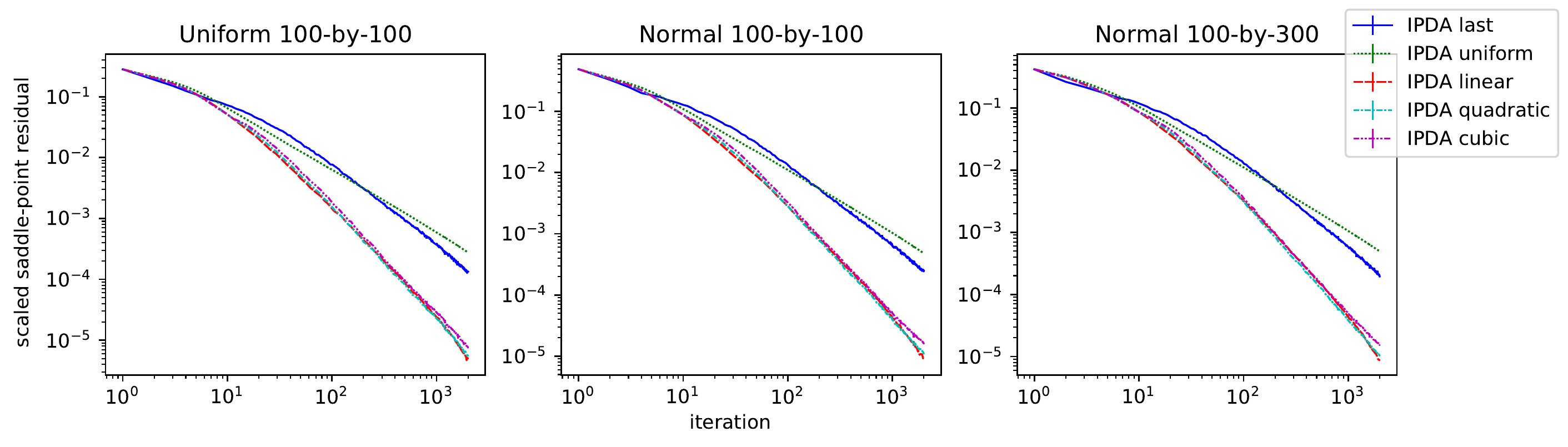}
		\includegraphics[width=0.5\columnwidth]{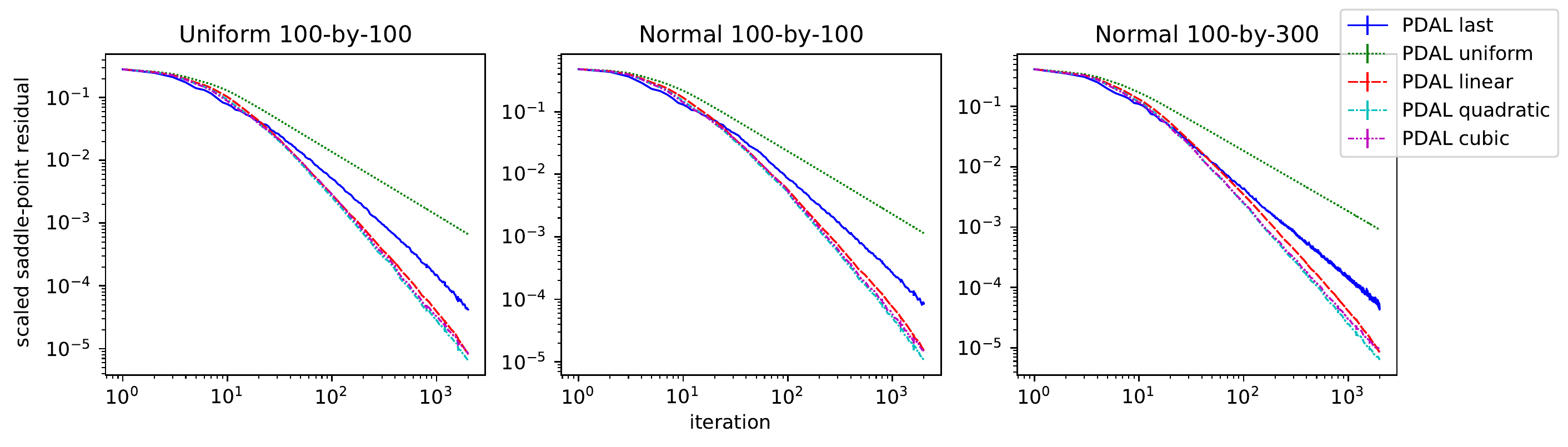}
		\includegraphics[width=0.5\columnwidth]{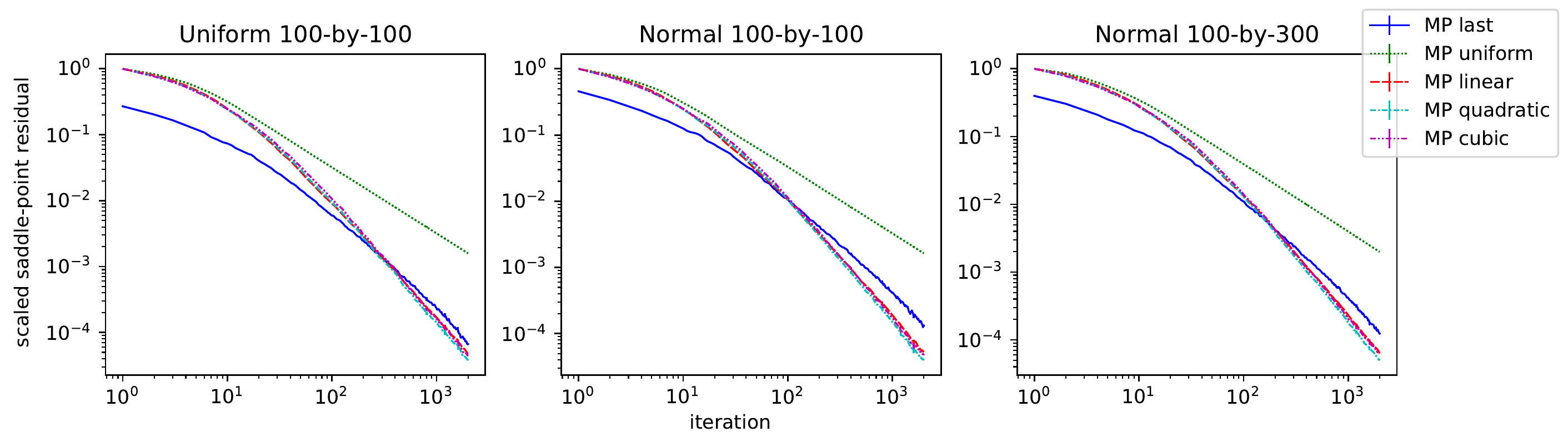} 			
		\includegraphics[width=0.5\columnwidth]{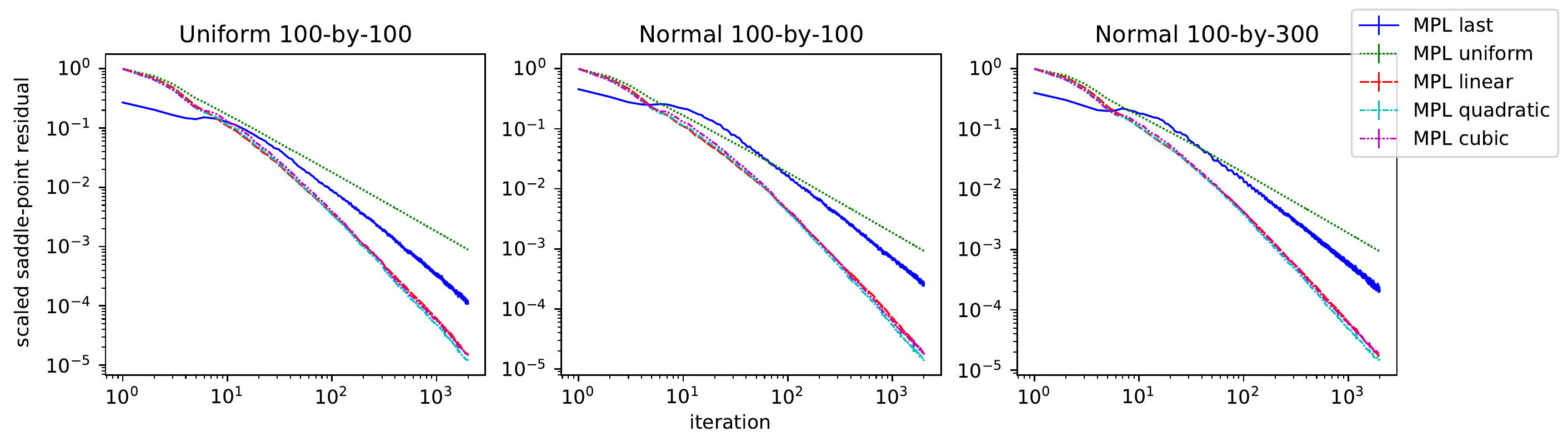}
		\caption{First-order methods with IIAS on matrix games}
		\label{fig:fom-ave-rand-mat}
		\vspace{-10px}
	\end{figure}
	
	\begin{figure}[htp]
		\begin{center}
			\includegraphics[width=0.7\columnwidth]{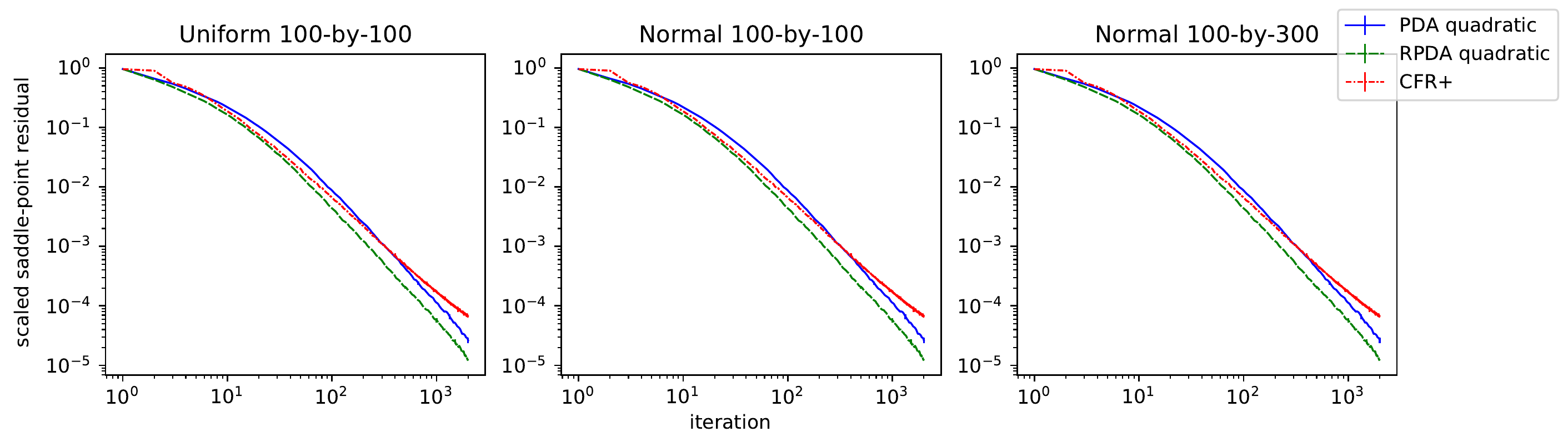} \includegraphics[width=0.5\columnwidth]{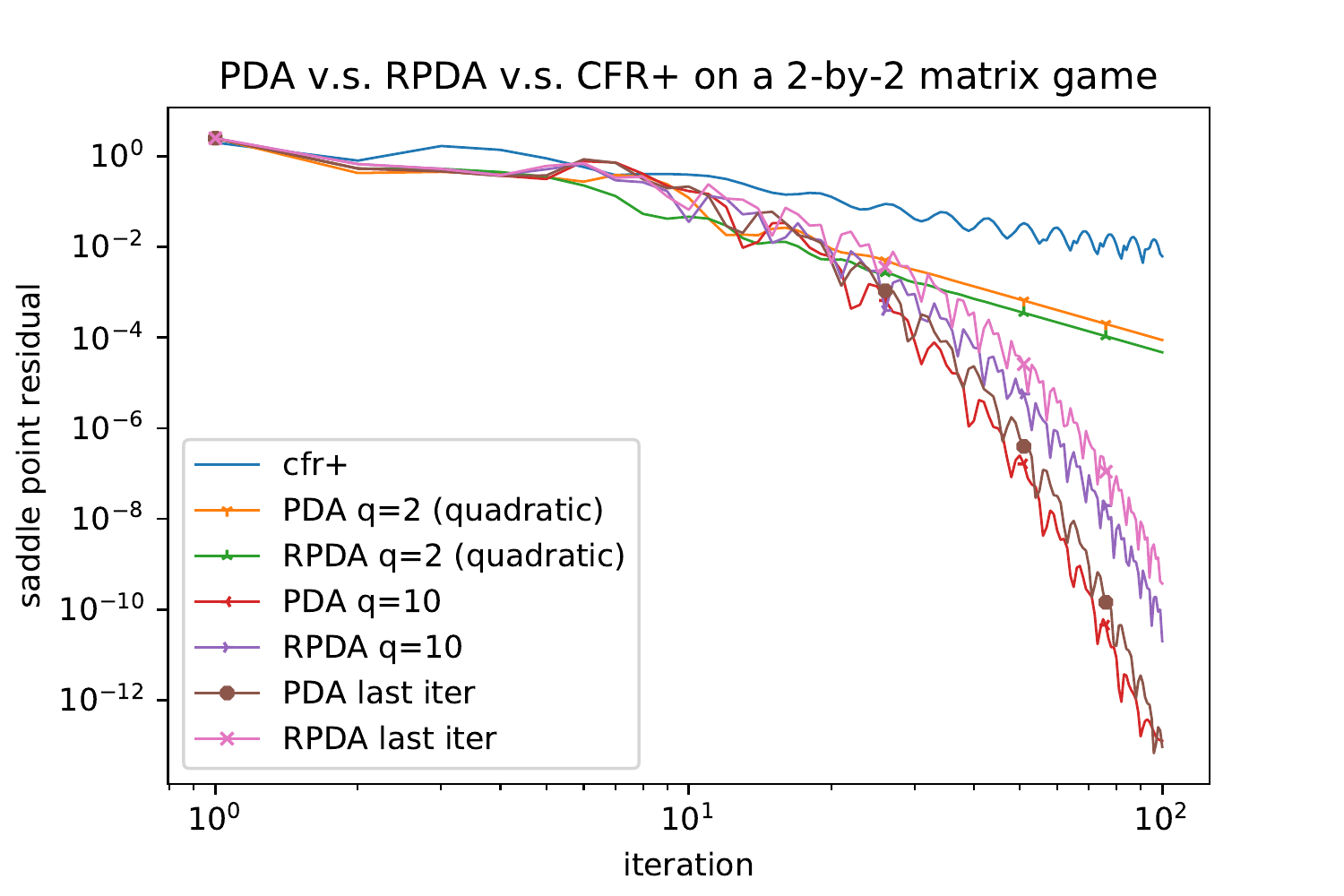}
			\caption{PDA and RPDA with IIAS v.s. CFR$^+$ for matrix games}
			\label{fig:fom-vs-cfr+}
		\end{center}
	\end{figure}
	\begin{figure}[htp]
		\begin{center}
			\includegraphics[width=0.47\columnwidth]{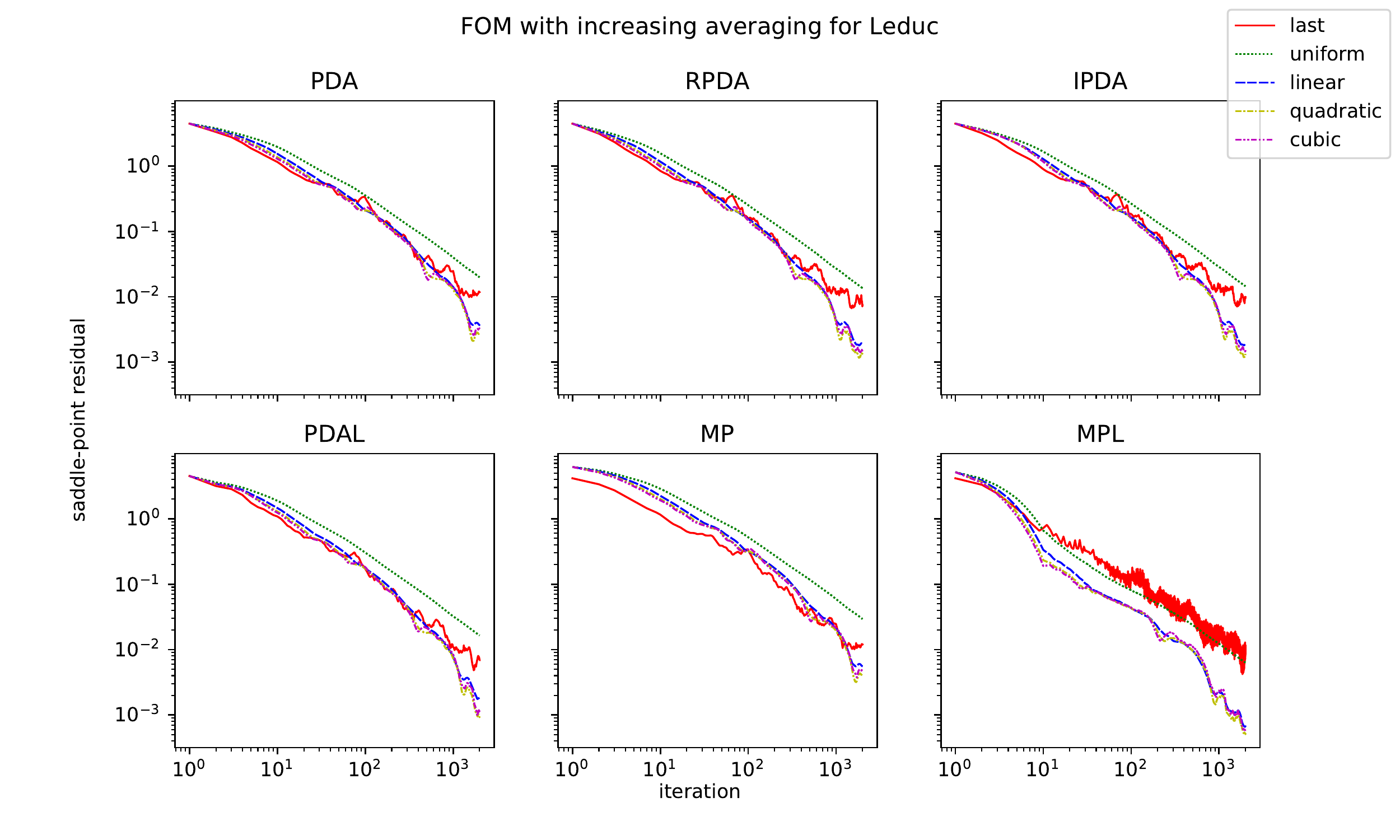}
			\includegraphics[width=0.47\columnwidth]{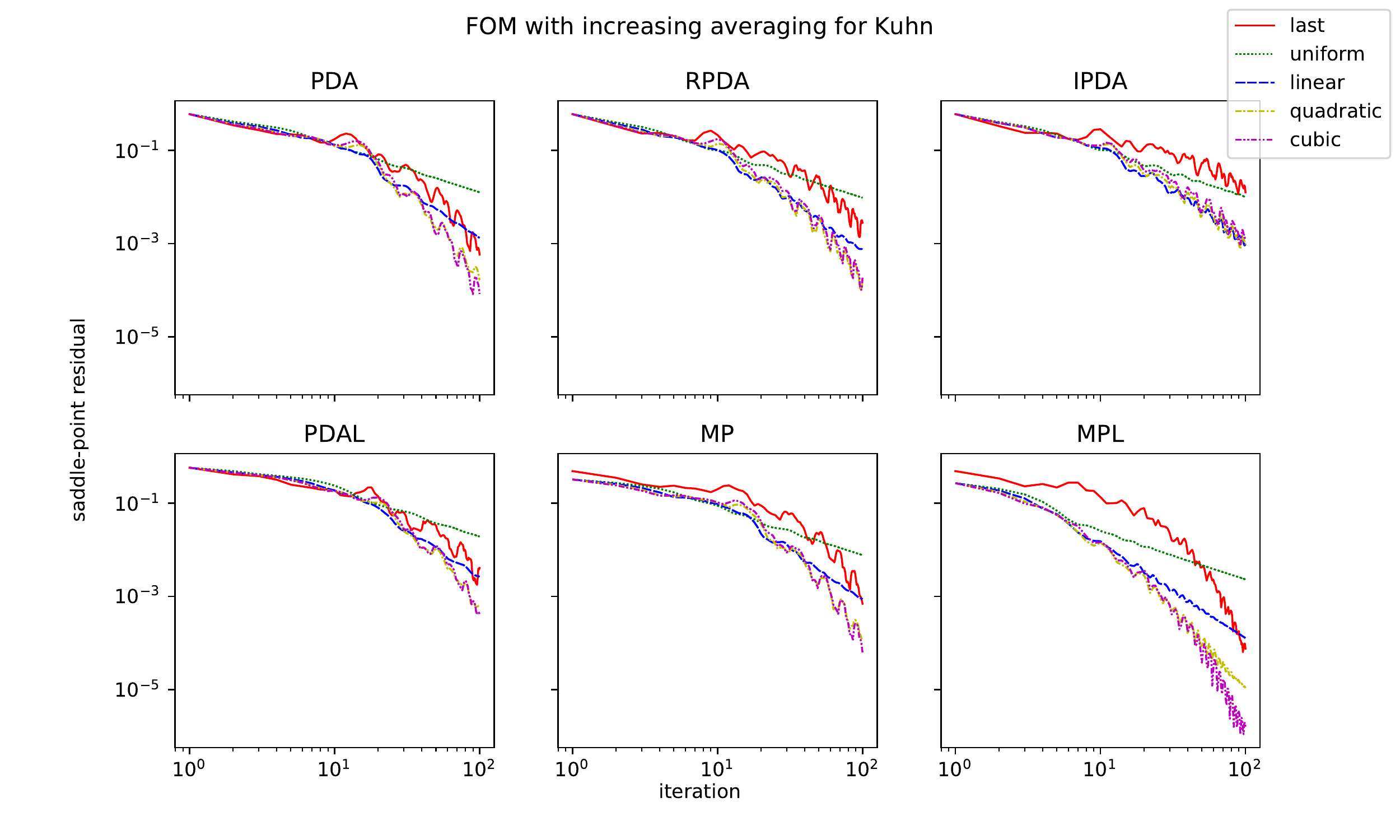}
				\vspace{-8px}
			\caption{First-order methods with IIAS for EFG}
			\label{fig:fom-with-ias-on-efg}
		\end{center}
	\vspace{-10px}
	\end{figure}

	\textbf{Market equilibrium }
	In a \textit{Fisher market}, each buyer $i$ has a \textit{valuation vector} $v_i \in \RR^m_+$ over $m$ goods. An \textit{allocation vector} $x_i\in \RR^m_+$ gives a utility of $v_i^\top x_i$ to buyer $i$. Each buyer $i$ has budget $B_i$ and each good $j$ has supply $s_j$. A \textit{competitive equilibrium} is a vector of prices $p\in \RR^m$ for the goods and an \textit{allocation} $x = [x_1, \dots, x_n]$ such that $x_i \in \argmax \left\{ v_i^\top x_i \mid p^\top x_i \leq B_i \right\}$ for all buyers $i$ and $\sum_{i} x_{ij} = s_j$ for all items $j$. To compute the equilibrium, it suffices to solve the \textit{Eisenberg-Gale convex program} \cite{eisenberg1959consensus, jain2010eisenberg}, which has a saddle-point formulation \cite{kroer2019computing}:
	\[ \min_{p\geq 0} \sum_{i=1}^n \max_{x_i\geq 0} \left[B_i\log (v_i^\top x_i) - p^\top x_i \right] + s^\top p.\] 
	We generate random instances of different sizes, solve them using PDA and compute the saddle-point residuals of the last iterates and various averages. Details of problem reformulation and random instance generation are deferred into Appendix \ref{app:detail-market-eq-exper}. Repeat each experiment $50$ times, normalize residuals and compute mean and standard deviations similar to the procedures in matrix games. Figure \ref{fig:pda-on-fisher} (left 3) displays the normalized residuals. The standard deviations have small magnitudes and thus become invisible. As the plots shows, increasing averages can converge as fast as, and sometimes even more rapidly ($q=5$ averaging in the rightmost subplot) than the last iterates.
	\begin{figure} 
		\centering
		\includegraphics[width=0.7\columnwidth]{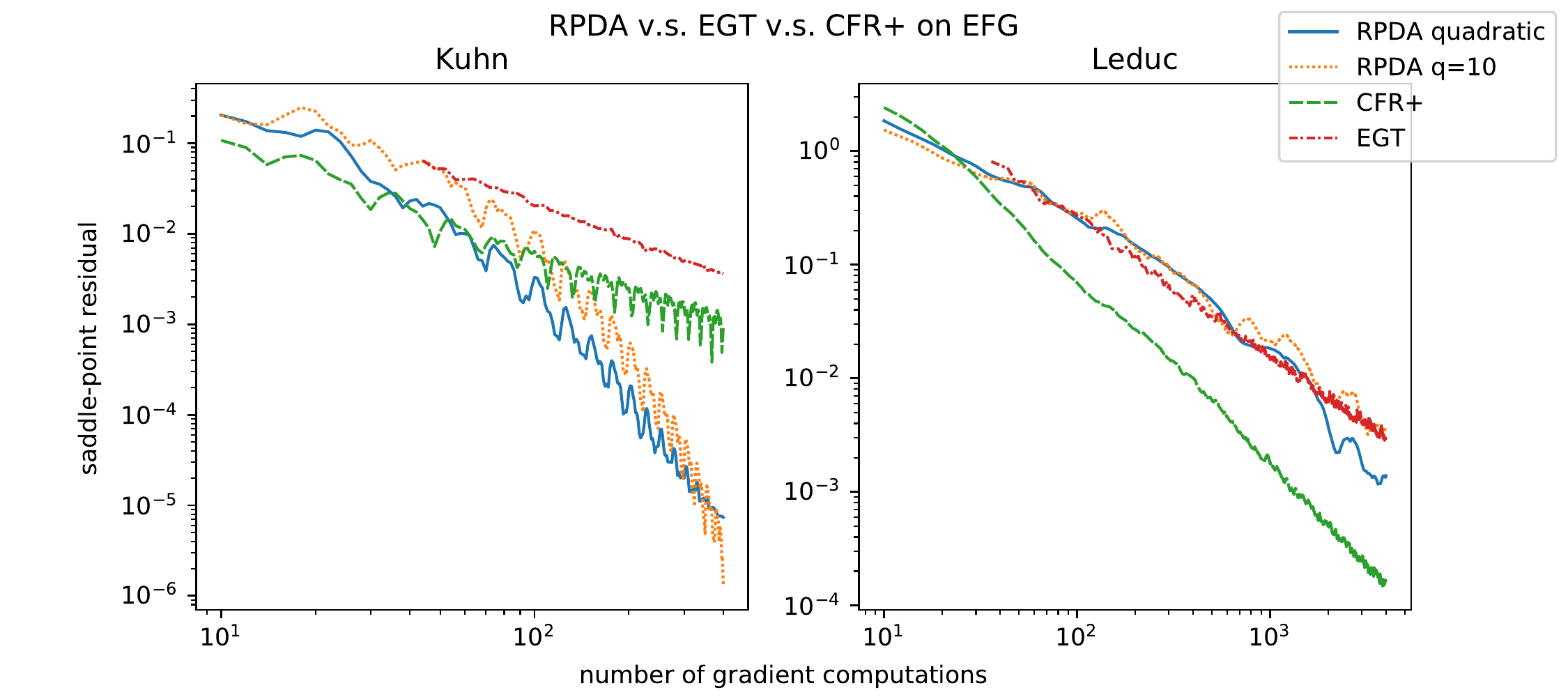}
		\caption{RPDA v.s. EGT v.s. CFR$^+$ for EFG}
		\label{fig:rpda-egt-cfr+-on-efg}
	\end{figure}

	\begin{figure}[htp]
		\begin{center}
			\includegraphics[width=0.48\columnwidth]{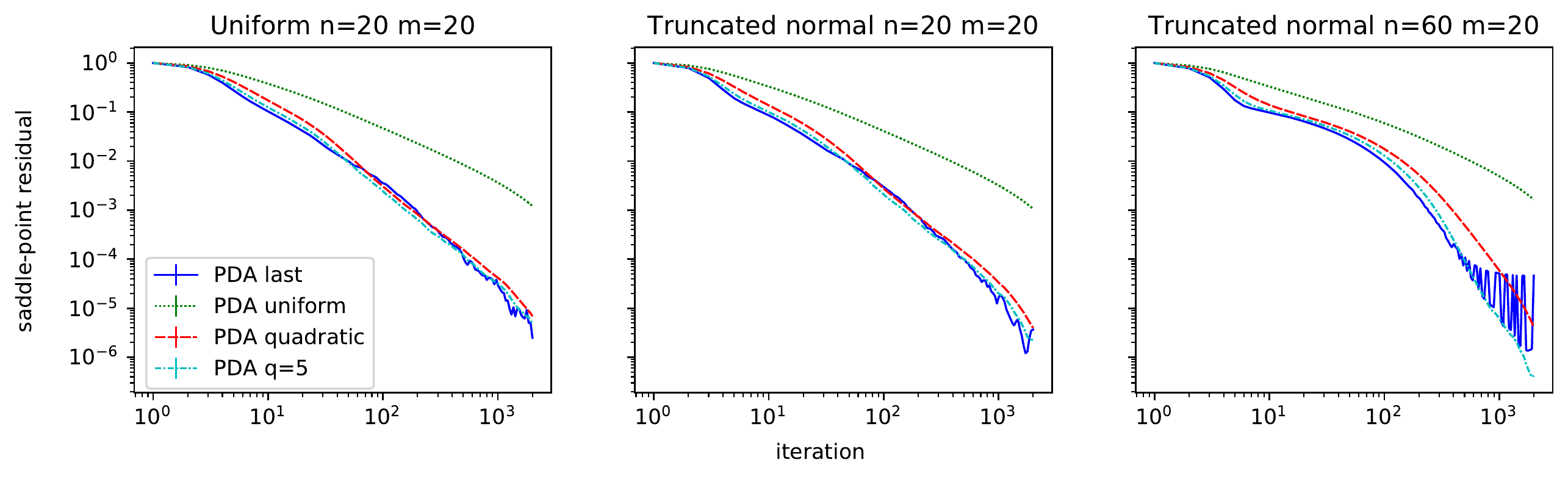} \includegraphics[width=0.48\columnwidth]{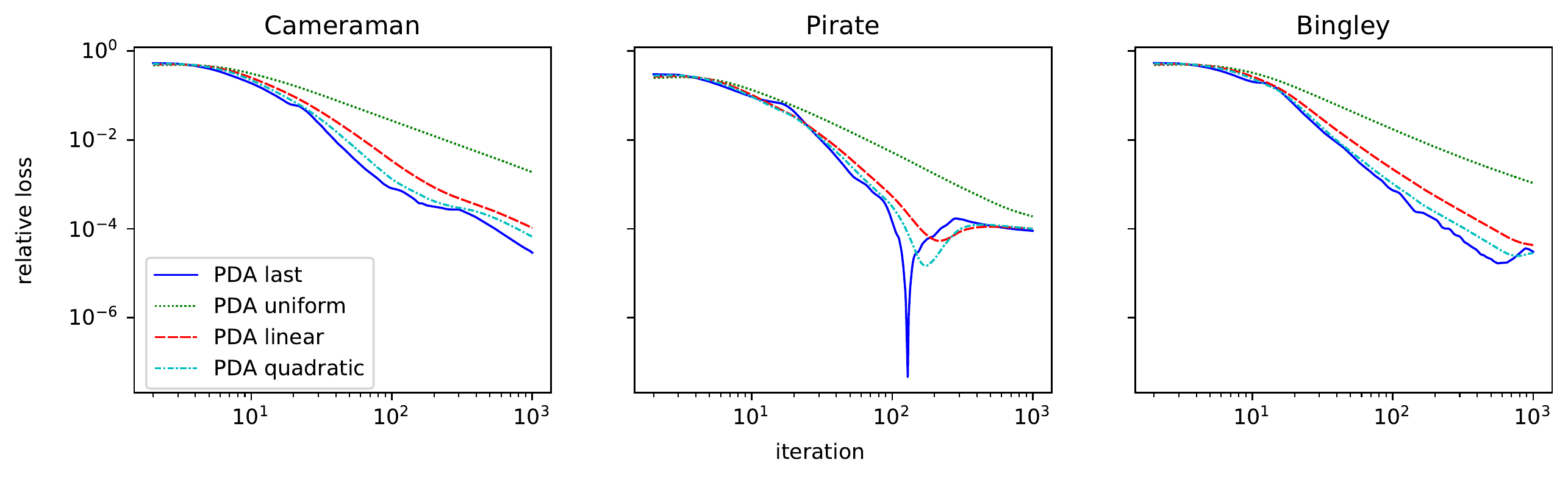}
			\caption{PDA with IIAS for equilibrium computation (left 3) and TV-$\ell_1$ minimization (right 3)}
			\label{fig:pda-on-fisher}\label{fig:pda-on-tv-l1}
		\end{center}
	\end{figure}

	\textbf{Image denoising via TV-$\ell_1$ minimization }
	The Total Variation (TV)-$\ell_1$ model is a means for image denoising through convex optimization \cite{chambolle2011first}. We use the saddle-point formulation of the convex optimization problem  \citet[pp. 132]{chambolle2011first}. Let $\bbX = \RR^{m\times n}$ be the image domain. Let $\divergence$ denote the divergence operator, that is, is the negative adjoint of the gradient operator $\nabla: \bbX \rightarrow \bbX$. Let $\bbY = X\times X = \RR^{m,n,2}$ be the set of discrete finite differences. Let $P = \{ p\in \bbY\mid (p^1_{ij})^2 + (p^2_{ij})^2 \leq 1 \}$ be the point-wise unit $\ell_2$-ball. The saddle-point formulation of the TV-$\ell_1$ model is
	\[\min_{u\in \mathcal{X}}\max_{p\in \cY} -\langle u, \divergence p\rangle + \lambda\|u-g\|_1 - \delta_P(p),\] 
	where $\lambda>0$ is the regularization strength hyperparameter. Following \citet{chambolle2011first}, to align it with \eqref{eq:sp-fgh*}, take $f=0$, $g(u) = \lambda \|u - g\|_1$ with $\lambda=1.5$ and $h^*(p) = \delta_P(p)$. In this way, the proximal mappings have closed-form formulas (see \cite[pp. 135-156]{chambolle2011first}). We add salt-and-pepper noise to three $256\times 256$-gray-scale images to obtain corrupted inputs and use the TV-$\ell_1$ minimization procedure for reconstruction. To solve the resulting saddle-point problems, we use PDA with default, static hyperparameters used in \cite{chambolle2011first} and run for $T=1000$ iterations. We compute the values of the original \textit{primal} TV-$\ell_1$ loss values of the last iterates and the various increasing averages. The loss values are normalized similarly and are displayed in Figure \ref{fig:pda-on-tv-l1} (right 3). Here, the last iterates converge fast, while linear and quadratic averages still perform nearly as well, and are far ahead of uniform averages. See Appendix \ref{app:reconstr-images} for the original, corrupted, and reconstructed (via PDA with quadratic averaging) images.\vspace{-2px}
	\section{Conclusion}\vspace{-2px}
	We proposed increasing iterate averaging schemes for various first-order methods, provided simple, implementable choices of averaging weights and established convergence properties of the new iterate averages. Extensive numerical experiments on various saddle-point problems demonstrated their ability to accelerate numerical convergence by orders of magnitude without modifying the original algorithm or incurring extra computation. We reiterate that the algorithms are run unaltered with untuned, theoretically safe hyperparameters, while the averaging weights are chosen simply based on their respective  convergence theorems. Even so, IIAS can bring the algorithms close to, and sometimes make them beat, other carefully engineered and tuned approaches.
	
%
	\bibliography{references}

\begin{thebibliography}{10}

\bibitem{beck2003mirror}
{\sc A.~Beck and M.~Teboulle}, {\em Mirror descent and nonlinear projected
  subgradient methods for convex optimization}, Operations Research Letters, 31
  (2003), pp.~167--175.

\bibitem{ben2019lectures}
{\sc A.~Ben-Tal and A.~Nemirovski}, {\em Lectures on modern convex
  optimization}, Online version: http://www2. isye. gatech. edu/\~{}
  nemirovs/Lect\_ ModConvOpt,  (2019).

\bibitem{bowling2015heads}
{\sc M.~Bowling, N.~Burch, M.~Johanson, and O.~Tammelin}, {\em Heads-up limit
  hold’em poker is solved}, Science, 347 (2015), pp.~145--149.

\bibitem{brown2018superhuman}
{\sc N.~Brown and T.~Sandholm}, {\em Superhuman ai for heads-up no-limit poker:
  Libratus beats top professionals}, Science, 359 (2018), pp.~418--424.

\bibitem{chambolle2011first}
{\sc A.~Chambolle and T.~Pock}, {\em A first-order primal-dual algorithm for
  convex problems with applications to imaging}, Journal of mathematical
  imaging and vision, 40 (2011), pp.~120--145.

\bibitem{chambolle2016ergodic}
\leavevmode\vrule height 2pt depth -1.6pt width 23pt, {\em On the ergodic
  convergence rates of a first-order primal--dual algorithm}, Mathematical
  Programming, 159 (2016), pp.~253--287.

\bibitem{davis2016convergence}
{\sc D.~Davis and W.~Yin}, {\em Convergence rate analysis of several splitting
  schemes}, in Splitting methods in communication, imaging, science, and
  engineering, Springer, 2016, pp.~115--163.

\bibitem{eisenberg1959consensus}
{\sc E.~Eisenberg and D.~Gale}, {\em Consensus of subjective probabilities: The
  pari-mutuel method}, The Annals of Mathematical Statistics, 30 (1959),
  pp.~165--168.

\bibitem{farina2019optimistic}
{\sc G.~Farina, C.~Kroer, and T.~Sandholm}, {\em Optimistic regret minimization
  for extensive-form games via dilated distance-generating functions}, in
  Advances in Neural Information Processing Systems, 2019, pp.~5222--5232.

\bibitem{gilpin2012first}
{\sc A.~Gilpin, J.~Pena, and T.~Sandholm}, {\em First-order algorithm with
  $\mathcal{O}(\ln(1/\epsilon))$-convergence for $\epsilon$-equilibrium in
  two-person zero-sum games}, Mathematical programming, 133 (2012),
  pp.~279--298.

\bibitem{golowich2020last}
{\sc N.~Golowich, S.~Pattathil, C.~Daskalakis, and A.~Ozdaglar}, {\em Last
  iterate is slower than averaged iterate in smooth convex-concave saddle point
  problems}, arXiv preprint arXiv:2002.00057,  (2020).

\bibitem{hoda2010smoothing}
{\sc S.~Hoda, A.~Gilpin, J.~Pena, and T.~Sandholm}, {\em Smoothing techniques
  for computing nash equilibria of sequential games}, Mathematics of Operations
  Research, 35 (2010), pp.~494--512.

\bibitem{jain2010eisenberg}
{\sc K.~Jain and V.~V. Vazirani}, {\em Eisenberg--gale markets: Algorithms and
  game-theoretic properties}, Games and Economic Behavior, 70 (2010),
  pp.~84--106.

\bibitem{juditsky2011first}
{\sc A.~Juditsky, A.~Nemirovski, et~al.}, {\em First order methods for
  nonsmooth convex large-scale optimization, i: general purpose methods},
  Optimization for Machine Learning,  (2011), pp.~121--148.

\bibitem{koller1996efficient}
{\sc D.~Koller, N.~Megiddo, and B.~Von~Stengel}, {\em Efficient computation of
  equilibria for extensive two-person games}, Games and economic behavior, 14
  (1996), pp.~247--259.

\bibitem{kroer2018solving}
{\sc C.~Kroer, G.~Farina, and T.~Sandholm}, {\em Solving large sequential games
  with the excessive gap technique}, in Advances in Neural Information
  Processing Systems, 2018, pp.~864--874.

\bibitem{kroer2019computing}
{\sc C.~Kroer, A.~Peysakhovich, E.~Sodomka, and N.~E. Stier-Moses}, {\em
  Computing large market equilibria using abstractions}, in Proceedings of the
  2019 ACM Conference on Economics and Computation, 2019, pp.~745--746.

\bibitem{kroer2018faster}
{\sc C.~Kroer, K.~Waugh, F.~K{\i}l{\i}n{\c{c}}-Karzan, and T.~Sandholm}, {\em
  Faster algorithms for extensive-form game solving via improved smoothing
  functions}, Mathematical Programming,  (2018), pp.~1--33.

\bibitem{malitsky2018first}
{\sc Y.~Malitsky and T.~Pock}, {\em A first-order primal-dual algorithm with
  linesearch}, SIAM Journal on Optimization, 28 (2018), pp.~411--432.

\bibitem{moravvcik2017deepstack}
{\sc M.~Morav{\v{c}}{\'\i}k, M.~Schmid, N.~Burch, V.~Lis{\`y}, D.~Morrill,
  N.~Bard, T.~Davis, K.~Waugh, M.~Johanson, and M.~Bowling}, {\em Deepstack:
  Expert-level artificial intelligence in heads-up no-limit poker}, Science,
  356 (2017), pp.~508--513.

\bibitem{nemirovski2004prox}
{\sc A.~Nemirovski}, {\em Prox-method with rate of convergence
  $\mathcal{O}(1/t)$ for variational inequalities with lipschitz continuous
  monotone operators and smooth convex-concave saddle point problems}, SIAM
  Journal on Optimization, 15 (2004), pp.~229--251.

\bibitem{nemirovsky1983problem}
{\sc A.~Nemirovski and D.~B. Yudin}, {\em Problem complexity and method
  efficiency in optimization.},  (1983).

\bibitem{nesterov2005excessive}
{\sc Y.~Nesterov}, {\em Excessive gap technique in nonsmooth convex
  minimization}, SIAM Journal on Optimization, 16 (2005), pp.~235--249.

\bibitem{pock2011diagonal}
{\sc T.~Pock and A.~Chambolle}, {\em Diagonal preconditioning for first order
  primal-dual algorithms in convex optimization}, in 2011 International
  Conference on Computer Vision, IEEE, 2011, pp.~1762--1769.

\bibitem{tammelin2015solving}
{\sc O.~Tammelin, N.~Burch, M.~Johanson, and M.~Bowling}, {\em Solving heads-up
  limit texas hold'em}, in Twenty-Fourth International Joint Conference on
  Artificial Intelligence, 2015.

\bibitem{yazici2018unusual}
{\sc Y.~Yazici, C.-S. Foo, S.~Winkler, K.-H. Yap, G.~Piliouras, and
  V.~Chandrasekhar}, {\em The unusual effectiveness of averaging in
  $\textnormal{GAN}$ training}, arXiv preprint arXiv:1806.04498,  (2018).

\end{thebibliography}
	\bibliographystyle{siam}
	
	\newpage
	\appendix
	
	\section*{Appendix}
	\section{Proofs and remarks}
	\subsection{Domain boundedness} \label{app:domain-bounded}
	We remark that the domain boundedness assumption is no more restrictive than that in \cite{chambolle2016ergodic}. For simplicity, consider the Euclidean setup for norms and Bregman divergences. Denote $z = (x,y)$. Then, $A(x,y,x',y') = \|z-z'\|_{M_{\tau, \sigma}}$. 
	The original bound in \cite[Theorem 1]{chambolle2016ergodic} is $\frac{\|z - z^0\|_{M_{\tau, \sigma}}}{T}$, while our bound in Theorem \ref{thm:pd-vanilla-weighted}, if we do not apply the final upper bound by $\Omega$, is 
	\[ \frac{\sum_{t=1}^T (w_t - w_{t-1}) \|z-z^{t-1}\|_{M_{\tau, \sigma}} }{S_T}. \]
  Thus, the difference is merely whether we measure the norm bound purely in terms of $z^0$ as in \cite{chambolle2016ergodic} or if we do it based on a weighted average of iterates as in our result.
	Since both bounds are for an arbitrary $z \in \dom g \times \dom h^*$, it is unclear how either expression can be upper bounded when the domain for $z$ is unbounded, using the current framework of analysis.
	Nevertheless,
	following \cite[Remark 3]{chambolle2016ergodic}, when $\tau\sigma\|K\|^2<1$ (which is true as long as one takes $L_f>0$, which is also w.l.o.g.), it holds that the sequence $z^n = (z^n, y^n)$ is globally bounded (say by $\|z^n - z^0\|\leq D$ for all $n$, where $\|z - z'\|^2 = \|x-x'\|_\bbX^2+\|y-y'\|_\bbY^2$) and converge to a saddle point $z^*$ (which also satisfies $\|z^* - z^0\|\leq D$). 
	Therefore, 
	the problem is equivalent to one with a bounded domain, that is, $z\in \dom g \times \dom h^*$ such that $\|z-z^0\|\leq D$. This reformulation leaves at least one saddle-point unaffected. We can then use this bounded reformulation to measure saddle-point residual, that is, $\max_z \left(\cL(x^t, y) -  \cL(y, x^t) \right)$ over all $z\in \dom g \times \dom h^*$ such that $\|z-z^0\|\leq D$ instead. With this reformulation, our bound involving $\Omega$ holds again.
	
%
%
%
	\subsection{Proof of Theorem \ref{thm:relax-pd-weighted}} \label{app:proof-relax}
	By the proof of Theorem 2 in \cite{chambolle2016ergodic}, under the said assumptions, for any $z=(x,y)\in \bbX\times\bbY$ and all $t = 0, 1, 2, \dots$, the following critical inequality holds:
	\begin{align}
	\cL(\xi^{t+1}, y) - \cL(x, \eta^{t+1}) \leq \frac{1}{2\rho_t} (A_t - A_{t+1}),
	\label{eq:relax-pd-crit-ineq}
	\end{align}
	where
	\begin{align}
	A_t := \frac{1}{2}\|z-z^t\|_{M_{\tau, \sigma}}^2 = \frac{1}{2 \tau}\|x - x^t\|^2 + \frac{1}{2 \sigma}\|y - t^t\|^2  - \langle K(x-x^t), y - y^t \rangle \leq \Omega.
	\end{align} 
	Subsequently, similar to the proof of Theorem \ref{thm:pd-vanilla-weighted}, the following telescoping sum bound holds:
	\begin{align}
	\sum_{t=1}^T w_t \left( \cL(\xi^t, y)-\cL(x, \eta^t) \right) \leq \frac{1}{2} \sum_{t=1}^T \left(\frac{w_t}{\rho_{t-1}} - \frac{w_{t-1}}{\rho_{t-2}}\right) A_{t-1} \leq \frac{1}{2}\sum_{t=1}^T \frac{w_t - w_{t-1}}{\rho_{t-1}}A_{t-1} \leq \frac{\Omega}{2\rho_0} w_T,
	\label{eq:relax-pd-telescoped}
	\end{align}
	where the first inequality is the weighted sum of \eqref{eq:relax-pd-crit-ineq} by $w_{t+1}$ over $t = 0, 1, \dots, T-1$ and the second is due to $0 < \rho_{t-2} \leq \rho_{t-1}$. The theorem follows in the same way as in Theorem \ref{thm:pd-vanilla-weighted}. \qed
	
	\subsection{Proof of Theorem \ref{thm:inertial-pd-weighted}}\label{app:proof-inertial}
	Recall that \[b_t = \min \left\{\frac{1-\alpha_{t-1}}{\alpha_t}, \frac{r(1-\alpha_{t-1}) - (1+\alpha_{t-1})}{\alpha_t (1 + 2r + \alpha_t)}\right\}, \ \ r = \frac{\frac{1}{\tau} - \sigma\|K\|^2}{L_f}\] 
	and $0\leq \alpha_{t-1} \leq \alpha_t \leq \alpha<1/3$ for all $t = 1, 2,\dots$ 
	First, we show the following algebraic lemma. 
	\begin{lemma}\label{lemma:inertial-M(n)-psd}
		Under the assumptions of Theorem \ref{thm:inertial-pd-weighted}, for all $t=1, 2, \dots$, it holds that $b_t > b^* > 1$ for some $b^*$ and
		\begin{align}
		1 - \alpha_{t-1} - \frac{w_{t+1}}{w_t} \alpha_t & \geq 0 \label{eq:inertial-ineq-denom-pos}, \\
		\left(\frac{1}{\tau} - \frac{1+\alpha_{t-1} + \frac{w_{t+1}}{w_t}(\alpha_t + \alpha_t^2)}{1  - \alpha_{t-1} - 2\frac{w_{t+1}}{w_t} \alpha_t}\right) \frac{1}{\sigma} & \geq \|K\|^2.  \label{eq:inertial-M(n)-psd}
		\end{align}
		In other words, the matrix $M_t := \begin{bmatrix}
		\left(\frac{1}{\tau} -  \frac{1+\alpha_{t-1} + (\alpha_t + \alpha_t^2)\frac{w_{t+1}}{w_t}}{1 - \alpha_{t-1} - 2 \frac{w_{t+1}}{w_t} \alpha_t } L_f  \right) I & - K^* \\ -K & \frac{1}{\sigma} I 
		\end{bmatrix} $ is positive semidefinite for all $t$. 
	\end{lemma}
	\begin{proof}
		Assume $L_f > 0$ without loss of generality. First, we show $b_t > 1$. Clearly, 
		\[ \frac{1-\alpha_{t-1}}{\alpha_t} \geq \frac{1 - \alpha_t}{\alpha_t} > 2. \]
		Meanwhile, 
		\begin{align*}
			& \left(\frac{1}{\tau} - \frac{(1+\alpha)^2}{1-3\alpha} L_f \right) \frac{1}{\sigma} > \|K\|^2 \\
			 & \Rightarrow r > \frac{(1+\alpha)^2}{1-3\alpha} \quad [*] \\& 
			\Rightarrow r(1-\alpha_{t-1}) - (1+\alpha_{t-1}) >  r(1-3 \alpha) - (1+\alpha)^2 \geq 0.
		\end{align*}
		By $0\leq \alpha_{t-1}\leq \alpha_t \leq \alpha<1/3$ and $[*]$,
		\begin{align*}
			& \frac{r(1-\alpha_{t-1}) - (1+\alpha_{t-1})}{\alpha_t(1 + 2r + \alpha_t)} \geq s:= \frac{r(1-\alpha) - (1+\alpha)}{2\alpha r + \alpha(1+ \alpha)}  \\ 
			& > \frac{ (1-\alpha)\cdot \frac{(1+\alpha)^2}{1-3\alpha} - (1+\alpha)}{2\alpha\cdot \frac{(1+\alpha^2)}{1 - 3\alpha} + \alpha(1+\alpha)} = \frac{(1-\alpha)(1+\alpha) - (1-3\alpha) }{2\alpha(1+\alpha) + \alpha(1-3\alpha)} = 1.
		\end{align*}
		Therefore, 
		\[b_t > b^*:= \min\left\{ 2, s \right\} > 1.\]
		Next, \eqref{eq:inertial-ineq-denom-pos} follows from $b_t \leq \frac{1-\alpha_{t-1}}{\alpha_t}$. Meanwhile, $b_t \leq \frac{r(1-\alpha_{t-1}) - (1+\alpha_{t-1})}{\alpha_t (1 + 2r + \alpha_t)}$ implies, via simple rearranging, 
		\[ \ell(b_t) :=  \frac{1+\alpha_{t-1} + b_t (\alpha_t + \alpha_t^2)}{1 - (\alpha_t + 2b_t \alpha_t)} \leq r = \frac{\frac{1}{\tau} - \sigma\|K\|^2}{L_f}, \]
		where $\ell(\cdot)$ is monotone increasing on $\left[1, \frac{1-\alpha_t}{2\alpha_t}\right]$. Therefore, the choice of the weights
		\[ w_{t+1} = w_t \cdot \min\left\{ b_t, \frac{(t+1)^q}{t^q} \right\} \]
		ensures $\frac{w_{t+1}}{w_t} \leq b_t$, which further implies $\ell(\frac{w_{t+1}}{w_t}) \leq \ell(b_t) \leq r$. This is just \eqref{eq:inertial-M(n)-psd} rearranged. 
	\end{proof}
	
	We then prove the theorem, that is, for all $(x, y)\in \bbX \times \bbY$,
	\begin{align}
		\cL(x, \bar{y}^T) - \cL(\bar{x}^T, y) \leq \frac{(1-\alpha_0)w_1 A_0 + \Omega \left[ (1-\alpha_{T-1})w_T + \alpha_1 w_2  - w_1 \right]}{S_T} \leq \frac{(q+1)(2-\alpha_0) \Omega }{T}. \label{eq:inertial-two-bounds}
	\end{align}
	First, $w_t$ is monotone increasing since $b_t > 1$ by Lemma \ref{lemma:inertial-M(n)-psd}.  Define $z^{-1} = z^0$. For $n \geq 0$, denote $A_t = \frac{1}{2}\|z-z^t\|^2_{M_{\tau,\sigma}}$, $B_t = \frac{1}{2}\|z^t - z^{t-1}\|^2_{M_{\tau,\sigma}}$ and $C_t = \frac{1}{2}\|x^t - x^{t-1}\|^2$. Note that $B_0 = C_0 = 0$ and $A_{-1} = A_0$. By the proof of Theorem 3 in \cite{chambolle2016ergodic}, the following critical inequality holds for any $t\geq 0$ and any $z = (x,y)\in \bbX \times \bbY$:
	\begin{align}
	& \cL(x^{t+1}, y) - \cL(x, y^{t+1})  \nnnl
	& \leq (A_t - A_{t+1}) + \alpha_t(A_t - A_{t-1}) + (\alpha_t-1) B_{t+1} + 2 \alpha_t B_t + L_f \left( (1+\alpha_t)C_{t+1} + (\alpha_t + \alpha_t^2) C_t \right). \label{eq:inertial-pd-criti-ineq}
	\end{align}
	
	Multiplying \eqref{eq:inertial-pd-criti-ineq} by $w_{t+1}$, summing over $t = 0, \dots, T-1$, and rearranging the right hand side yield
	\begin{align}
	& \sum_{t=1}^{T} w_t(\cL(x^t, y) - \cL(x, y^t)) \nnnl 
	& \leq \sum_{t=1}^{T-1} (w_{t+1} - w_t)A_t + ( w_1 A_0 - w_T A_T) \nnnl 
	& \quad + \sum_{t=0}^{T-2} (\alpha_t w_{t+1} - \alpha_{t+1}w_{t+2}) A_t + ( - \alpha_0 w_1 A_{-1}  + \alpha_{T-1} w_T A_{T-1}) \nnnl 
	& \quad + \sum_{t=1}^{T-1} \left[  \left((\alpha_{t-1} - 1)w_t + 2\alpha_t w_{t+1} \right) B_t + L_f\left((1+\alpha_{t-1}) w_t + (\alpha_t + \alpha_t^2) w_{t+1}  \right) C_t \right] \nnnl
	& \quad + (\alpha_{T-1} - 1) w_T B_T + 2 \alpha_0 w_1 B_0 + L_f \left[ (1+\alpha_{T-1} )w_T C_T + (\alpha_0 + \alpha_0^2) w_1 C_0 \right]. \label{eq:inertial-weighted-sum-ineq}
	\end{align}
	Then, we simplify and bound the summation terms and ``leftover'' terms separately. First, recall that $\alpha_0 w_1 B_0 = 0$ and $(\alpha_0 + \alpha_0^2) w_1 C_1 = 0$. Since $A_t \leq \Omega$ for all $t$, we have 
	\begin{align}
	\sum_{t=1}^{T-1} (w_{t+1} - w_t)A_t +  \sum_{t=1}^{T-2} (\alpha_t w_{t+1} - \alpha_{t+1}w_{t+2}) A_t  \leq \Omega\left[ (1-\alpha_{T-1})w_T + \alpha_1 w_2  - w_1 \right]. \label{eq:inertial-A(t)-bound}
	\end{align}
	Meanwhile, straightforward computation verifies that
	\begin{align}
	\sum_{t=1}^{T-1} \left[ \left((\alpha_{t-1} - 1)w_t + 2\alpha_t w_{t+1} \right) B_t + L_f \left((1+\alpha_{t-1}) w_t + (\alpha_t + \alpha_t^2) w_{t+1}  \right) C_t \right]  \nnnl
	=  - \sum_{t=1}^{T-1} w_t \left(1 - \alpha_{t-1} - 2\frac{w_{t+1}}{w_t}\alpha_t\right) \| z^t - z^{t-1}\|^2_{M_t}  \leq 0, \label{eq:inertial-M(n)-bound}
	\end{align}
	where the last inequality is due to \eqref{eq:inertial-ineq-denom-pos} and 
	\eqref{eq:inertial-M(n)-psd} in Lemma \ref{lemma:inertial-M(n)-psd}.
	
	Using the inequality $\frac{1}{2}\|a+b\|^2 \leq \|a\|^2 + \|b\|^2$, we have \[\frac{1}{2}A_{T-1} \leq A_T + B_T.\]
	Therefore, 
	\begin{align}
	\alpha_{T-1} w_T A_{T-1} - w_T A_T + (\alpha_{T-1}-1) w_T B_T + L_f (1+\alpha_{T-1}) w_T C_T \nnnl
	\leq w_T\left[ (2\alpha_{T-1}-1) A_T + (3\alpha_{T-1}-1) B_T + L_f(\alpha_{T-1} + 1)C_T \right] \leq 0, \label{eq:inertial-leftover-terms-bound}
	\end{align}
	where the last inequality follows from $\frac{2\alpha_{T-1}-1}{2} < 0$ and 
	\[(3\alpha_{T-1}-1)B_T + L_f (\alpha_{T-1}+1) C_T = - (1- 3\alpha_{T-1})\|z - z^T\|_P \leq 0, \]
	where $P := \begin{bmatrix}
	(\frac{1}{\tau} - \frac{1+\alpha_{T-1}}{1-3\alpha_{T-1}} L_f)I & -K^* \\ -K & \frac{1}{\sigma}I 
	\end{bmatrix}$ is positive semi-definite since, by $\alpha_{T-1}\leq \alpha < 1/3$,
	\[ \left( \frac{1}{\tau} - \frac{1+\alpha_{T-1}}{1 - 3\alpha_{T-1}}L_f \right) \frac{1}{\sigma} \geq \left( \frac{1}{\tau} - \frac{(1+\alpha)^2}{1 - 3\alpha}L_f \right)\frac{1}{\sigma} \geq \|K\|^2. \]
		
	Now, the only ``untreated'' terms on the right hand side of \eqref{eq:inertial-weighted-sum-ineq} are $w_1 A_0$ and $-\alpha_0 w_1 A_{-1} = -\alpha_0 w_1 A_0$. Combining \eqref{eq:inertial-weighted-sum-ineq}, \eqref{eq:inertial-A(t)-bound}, \eqref{eq:inertial-M(n)-bound} and \eqref{eq:inertial-leftover-terms-bound}, we have
	\begin{align*}
	\sum_{t=1}^T w_t ( \cL(x^t, y) - \cL(x, y^t) ) \leq (1-\alpha_0)w_1 A_0 + \Omega \left[ (1-\alpha_{T-1})w_T + \alpha_1 w_2  - w_1 \right].
	\end{align*} 
	The first inequality in \eqref{eq:inertial-two-bounds}
	follows from the convex-concave structure of $\cL$.  
	Furthermore,
	 \[ (1-\alpha_0)w_1 A_0 + \Omega \left[ (1-\alpha_{T-1})w_T + \alpha_1 w_2  - w_1 \right] \leq (1-\alpha_0)\Omega + \Omega w_T \leq (2 - \alpha_0) w_T \Omega.\]
	By the choice of $w_t$, for $1\leq t\leq T-1$, one has $1\leq \frac{w_T}{w_t} \leq \frac{(t+1)^q}{t^q}\cdots \frac{T^q}{(T-1)^q} = \frac{T^q}{t^q}$. Therefore, 
	\begin{align*}
	\frac{w_T}{S_T} & = \frac{1}{\sum_{t=1}^T \frac{w_t}{w_T}} \leq \frac{1}{ \sum_{t=1}^T \frac{t^q}{T^q} } = \frac{T^q}{\sum_{t=1}^T t^q} \leq \frac{q+1}{T}
	\end{align*}  
	and the second inequality in \eqref{eq:inertial-two-bounds} follows.
	\qed 
	
	\subsection{Proof of Theorem \ref{thm:pdal}}\label{app:proof-pdal}
	For any $(x', y')\in \bbX \times \bbY$, denote 
	\begin{align*}
	P(x') &= g(x') - g(x) + \langle K^* y,x' - x \rangle, \\
	D(y') &= h^*(y') - h^*(y) - \langle K x, y' - y\rangle.
	\end{align*}
	Note that $P(\cdot)$ and $D(\cdot)$ are both convex. Denote 
	\[E_t = \frac{1}{2}(\|x^t - x\|^2 + \frac{1}{\beta} \|y^{t-1} - y\|^2 ),\] 
	which clearly satisfies \[E_t \leq \Omega_\bbX + \frac{1}{\beta}\Omega_\bbY\] 
	by assumption. 
	By the proof of Theorem 3.4 in \cite{malitsky2018first}, the following critical inequality hold for any $t\geq 1$ and any $(x, y)\in \bbX \times \bbY$:\footnote{In \citet{malitsky2018first}, the authors assume $(x,y)$ is a saddle point of \eqref{eq:sp-fgh*} (with $f = 0$). However, the statement and proof of their Theorem 3.5 in fact hold for general $(x,y)$. The saddle point assumption is needed in establishing sequence convergence \cite[Theorem 3.4]{malitsky2018first}.}
	\begin{align}
	\tau_t \left[ (1+\theta_t)P(x^t) - \theta_t P(x^{t-1}) + D(y^t) \right] \leq E_t - E_{t+1}. 
	\label{eq:crit-ineq-pdal}
	\end{align}
	Multiplying both sides by $w_t$ and summing up over $t = 1, \dots, T$ yield 
	\begin{align}
	A + B \leq \sum_{t=1}^T w_t (E_t - E_{t+1}) \leq \sum_{t=1}^T (w_t - w_{t-1}) E_t \leq \left(\Omega_\bbX + \frac{1}{\beta}\Omega_\bbY\right) w_T, \label{eq:pdal-proof-A+B<=}
	\end{align}
	where
	\begin{align*}
	A &= \sum_{t=1}^T w_t \tau_t\left[ (1+\theta_t) P(x^t) - \theta_t P(x^{t-1}) \right] \nnnl
	&= - w_1 \tau_1 \theta_t P(x^0) + w_T \tau_T (1+ \theta_T)P(x^T) + \sum_{t=2}^T \left[ (1+\theta_{t-1}) w_{t-1}\tau_{t-1} - \theta_t w_t \tau_t \right] P(x^{t-1}), \nnnl
	B &= \sum_{t=1}^T w_t \tau_t D(y^t).
	\end{align*}
	
	By the algorithm, $\theta_t = \frac{\tau_t}{\tau_{t-1}}$. By the choice of $w_t$, we have $\frac{w_t}{w_{t-1}} \leq \frac{1+\theta_{t-1}}{\theta_t^2}$. Therefore, it holds that 
	\begin{align}
	(1 + \theta_{t-1}) w_{t-1}\tau_{t-1} \geq \theta_t^2 w_t \tau_{t-1} \geq \theta_t w_t \tau_t. \label{eq:pdal-proof-nonneg-coeff}
	\end{align}
	By the definition of $\tilde{x}^t$ and rearrangement of terms, it holds that
	\begin{align*}
	& w_T \tau_T (1+ \theta_T) + \sum_{t=2}^T \left[ (1+\theta_{t-1}) w_{t-1}\tau_{t-1} - \theta_t w_t \tau_t \right] = w_1 \tau_1 \theta_1 + S_T \nnnl
	& w_T \tau_T (1+ \theta_T) x^T + \sum_{t=2}^T \left[ (1+\theta_{t-1}) w_{t-1}\tau_{t-1} - \theta_t w_t \tau_t \right]x^{t-1} = w_1 \theta_1 \tau_1 x^0 + \sum_{t=1}^T w_t \tau_t \tilde{x}^t.
	\end{align*}
	By \eqref{eq:pdal-proof-nonneg-coeff} (which ensures the coefficients of $x^{t-1}$  in the above summation are nonnegative), convexity of $P(\cdot)$ and the above identities, one has
	\begin{align}
	A \geq \left(w_1 \tau_1 \theta_1 + S_T \right) P(\bar{x}^T) - w_1 \tau_1 \theta_1 P(x^0) \geq S_T P(\bar{x}^T) - w_1 \tau_1 \theta_1 P(x^0). \label{eq:pdal-proof-A>=...}
	\end{align}
	Meanwhile, convexity of $D(\cdot)$ implies
	\begin{align}
	B\geq S_T D(\bar{y}^T). \label{eq:pdal-proof-B>=...}
	\end{align}
	Finally, substituting \eqref{eq:pdal-proof-A>=...} and \eqref{eq:pdal-proof-B>=...} into \eqref{eq:pdal-proof-A+B<=} gives ($P_0 = P(x^0)$)
	\[ \cL(\bar{x}^T, \bar{y}^T) = P(\bar{x}^T) + D(\bar{y}^T) \leq \frac{\left(\Omega_\bbX + \frac{1}{\beta}\Omega_\bbY\right) w_T + w_1 \tau_t \theta_t P_0}{S_T}, \]
	proving the inequality on the left. The one on the right can be seen via $\frac{w_T}{S_T} \leq \frac{q+1}{T}$ (see Appendix \ref{app:proof-inertial}).
	 \qed
	
	\subsection{Proof of Theorem \ref{thm:mp}}\label{app:proof-md}
	We only need to prove the second part when $\tau_t = \frac{1}{T}$, since the rest of the theorem is Theorem 5.6.1 in \cite{ben2019lectures}. By this theorem, $\tau_t = \frac{1}{T}$ implies $\delta_t \leq 0$. Therefore, the right hand side 
	\[ \phi(\bar{x}^T, y) - \phi(x, \bar{y}^T) \leq \frac{w_T \Omega + \sum_{t=1}^T w_t \delta_t }{S_T} \]
	reduces to $\frac{w_T \Omega}{S_T}$. By the choice of $w_t$, we have $\frac{w_T}{S_T}  = \frac{L T^q}{\sum_{t=1}^T t^q}\leq \frac{(q+1)L}{T}$, where the last inequality is the same as in the proof of Theorem 1. 
	
	\section{Further details on numerical experiments}\label{app:details-num-exper}
	\subsection{Matrix game and extensive-form game experiments}\label{app:details-game-exper}
	\textbf{Random matrix generation\ } For random matrix games, matrices have i.i.d. entries distributed as either $\frac{1}{2}{\mathcal{U}(0,1)}-1$ (``uniform'') or $\mathcal{N}(0,1)$ (''normal''), where $\mathcal{U}(0,1)$ is the uniform distribution on $[0,1]$ and $\mathcal{N}(0,1)$ is the standard normal distribution. We generate $100\times 100$-uniform, $100\times 100$-normal and $100\times 300$-normal random matrices, as shown in Figure \ref{fig:fom-ave-rand-mat}.
	
	\textbf{Algorithm hyperparameters} For both random matrix games and EFG, the hyperparameters of the algorithms are set as follows. for PDA, RPDA, IPDA, $\tau = \sqrt{\frac{1 - 1/n_2}{1 - 1/n_1}} \alpha$, $\sigma =  \sqrt{\frac{1 - 1/n_1}{1 - 1/n_2}} \alpha$, where $\alpha = 0.99/L$ and $L$ is the oeprator norm of $K=A^\top$. for RPDA, the relaxation parameter is fixed at $\rho_t = \rho = 1.5$; for IPDA, set $\alpha_t = \alpha = 0.3$; for PDAL, set $\mu = 0.2$, $\delta = 0.8$ and $\beta=1$ throughout \cite{malitsky2018first}. For MP, $\tau_t = \tau = 1/L$. For MPL, using the notation on page 443 of \cite{ben2019lectures}, further set the aggressive stepsizing multipliers to $\theta^+ = 1.2$, $\theta^- = 0.8$, and $\tau_{\rm safe} = 1/L$ (we use $\tau$ for stepsize, which is $\gamma$ therein).
	
	For IPDA, sice $f=0$, we have $L_f = 0$. Therefore, by the choice of $b_t$ in Theorem \ref{thm:inertial-pd-weighted}, we can take $b_t = \frac{1-\alpha}{2\alpha} = 7/6$. Then, we choose $w_t$ according to the recursive formula \ref{eq:inertial-pd-weights-choice}. Note that this does not affect the eventual polynomial growth of $w_t$. For example, for quadratic averaging, since $(16/15)^2 <  7/6$, we have $\frac{w_{t+1}}{w_t} = \frac{(t+1)^2}{t^2}$for all $t\geq 15$. For PDAL, the weights are chosen according to the formula in Theorem \ref{thm:pdal}, that is,
	$w_{t+1} = w_t \cdot \min\left\{ \frac{1+\theta_t}{\theta_{t+1}}, \frac{(t+1)^q}{t^q} \right\}$
	for all values of $q$. Note that they are nondecreasing but may not necessarily grow polynomially eventually, unlike IPDA.
	
	\textbf{Normalizing the saddle-point residuals } The randomly generated matrix games can have varying difficulty - some are much easier to solve (that is, to get a solution with very small saddle-point residual) than others. Therefore, we normalize the saddle-point residuals across different matrix games as follows. For each generated matrix game and each algorithm, let the computed saddle-point residuals of the increasing averages be $\epsilon^q_1, \dots, \epsilon^q_{T}$ (also compute them for the last iterates, ``$q = {\rm last}$''). in order to make the sample paths comparable, we normalize each residual $\epsilon^q_t$ to $\frac{\epsilon^q_t - \epsilon_*}{\epsilon^*}$, where $\epsilon_* := 0.5\cdot \min_{q, t} \epsilon^q_t$ and $\epsilon^* = \max_q \epsilon^q_1$. In this way, we are comparing the \textit{relative} performance of different algorithms with different averaging schemes. 
	
	For the game \verb|Kuhn|, each algorithm is run for $T=100$ iterations; for \verb|Leduc|, each algirhtm is run for $T=2000$ iterations.

	\subsection{Numerical stability of iterative average updates} \label{app:detail-large-q}
	In our Python implementation, the weights and averages are maintained as follows. Suppose we use PDA with IIAS, where the weight exponent is $q$. At iteration $t$, we have kept the previous average \verb|x_ave| (i.e., $\bar{x}_{t-1}$), previous weight sum \verb|S| (i.e., $S_{t-1}:= \sum_{\tau=1}^{t-1}w_\tau^q$) and just computed the current iterate \verb|x| (i.e., $x_t$). The average and sum of weights are updated as follows: 
	\begin{verbatim}
	w = t**q
	S_new = S + w; 
	x_ave[t] = x_ave[t-1] * S / S_new + x[t] * w / S_new
	S = S_new
	\end{verbatim}
	When $t = 4000$ and $q = 10$ (both are larger than any possible scenario in our experiments), we can emulate the above (only computing the weights and weight sum, no iterate average update) via the following code snippet. 
	\begin{verbatim}
	T = 4000
	S = 0
	q = 10
	for t in range(1, T+1):
	w = t ** q
 	S_new = S + w
	if t == T:
	print("S/S_new = {}, w/S_new = {}".format(S/S_new, w/S_new))
	S = S_new
	\end{verbatim}
	Using Python 3 in Google Colab, it prints the ratios in the last iteration, which are roughly 0.99725 and 0.00275. These numbers are of reasonable magnitude and do not lead to numerical issues. Again, we emphasize that $q=1,2$ already brings significant numerical speedup and these are the values that we suggest practitioners to use.
	
	\subsection{Details of the equilibrium computation experiments} \label{app:detail-market-eq-exper}
	\textbf{Reformulation of the Eisenberg-Gale saddle-point formulation into \eqref{eq:sp-fgh*}\ } The reformulation here is based on \cite{kroer2019computing}. They show that, in equilibrium, each buyer $i$ is guaranteed at least their \textit{proportional} allocation under any feasible set of prices. Therefore, for each buyer $i$, there exists a \textit{proportional shares} $\gamma_i>0$ such that $v_i^\top x_i \geq \gamma_i$ under equilibrium. Denote the resulting set of allocation vectors for buyer $i$ as $U_i = \{ u\in \RR^m\mid v_i^\top u \geq \gamma_i,\, u \geq 0 \}$ and denote the indicator function of a set $C$ as $\delta_{C}\in \{0, \infty\}$. Then, the saddle-point problem is equivalent to
	\[ 	\min_{p\geq 0} \left[\sum_{i=1}^n \max_{x_i\in U_i} \left(B_i\log (v_i^\top x_i) - p^\top x_i\right) + s^\top p\right]. \]
	
	Note that we can impose upper bounds on both $x$ and $p$ without affecting the equilibrium. in fact, any equilibrium solution must satisfy $0 \leq x_i \leq s$ and $0 \leq p \leq \bar{p}$ for some $\bar{p}>0$ \cite{kroer2019computing}. Therefore, by Sion's minmax theorem, we can interchange min and max:
	\[ \max_{x_i \in U_i} \min_{p\geq 0} \left[ \sum_{i=1}^n \left( B_i \log(v_i^\top x_i) - p^\top x_i  \right) + s^\top p \right]. \]
	
	To align with \eqref{eq:sp-fgh*}, simply negating the signs:
	\[ \min_{x_i \in U_i} \max_{p\geq 0} \left[ -\sum_{i=1}^n \left( B_i \log(v_i^\top x_i) - p^\top x_i  \right) - s^\top p \right], \]
	where we set \[f(x) = -\sum_{i=1}^n B_i \log(v_i^\top x_i),\ g(x) = \sum_i \delta_{U_i}(x_i)\ h^*(p) = -s^\top p,\] 
	and $K$ to be the matrix encoding \[(x,p)\mapsto \sum_i p^\top x_i.\] 
	Here, indeed, $f$ is convex and smooth, whose Lipschitz constant can be bounded by $\sqrt{n} + \max_{i,j} v_{ij} \frac{B_i}{\gamma_i}$ \cite{kroer2019computing}. We then use PDA to solve the reformulated problem.
	
	\textbf{Random instance generation\ } We generate 3 types of random instances: $(n,m) = (20, 20)$ with i.i.d. $\mathcal{U}(0,1)$ valuations, $(20,20)$ with i.i.d. truncated normal valuations, and $(40, 20)$ with i.i.d. truncated normal ($\mu = 5, \sigma^2 = 4$ truncated into $[0,10]$) valuations.
	
	\subsection{Preconditioned PDA and RPDA with IIAS} \label{app:precond-rpda-cfr}
	For both PDA and RPDA, we use the same parameters as described in Appendix \ref{app:details-game-exper} and the preconditioner construction in \cite[\S Lemma 2]{pock2011diagonal}. In theory, this does not alter the cost per treeplex projection \cite[Theorem 5]{gilpin2012first}. The preconditioned versions are named PPD and PRPD, respectively. As can be seen, preconditioning brings additional speedup to the original algorithms using unweighted Euclidean norms. In particular, PRPD with various IIAS can catch up with CFR$^+$ toward the end of $T=4000$ iterations.
	
		\begin{figure}[htp]
		\begin{center}
			\includegraphics[width=0.48\columnwidth]{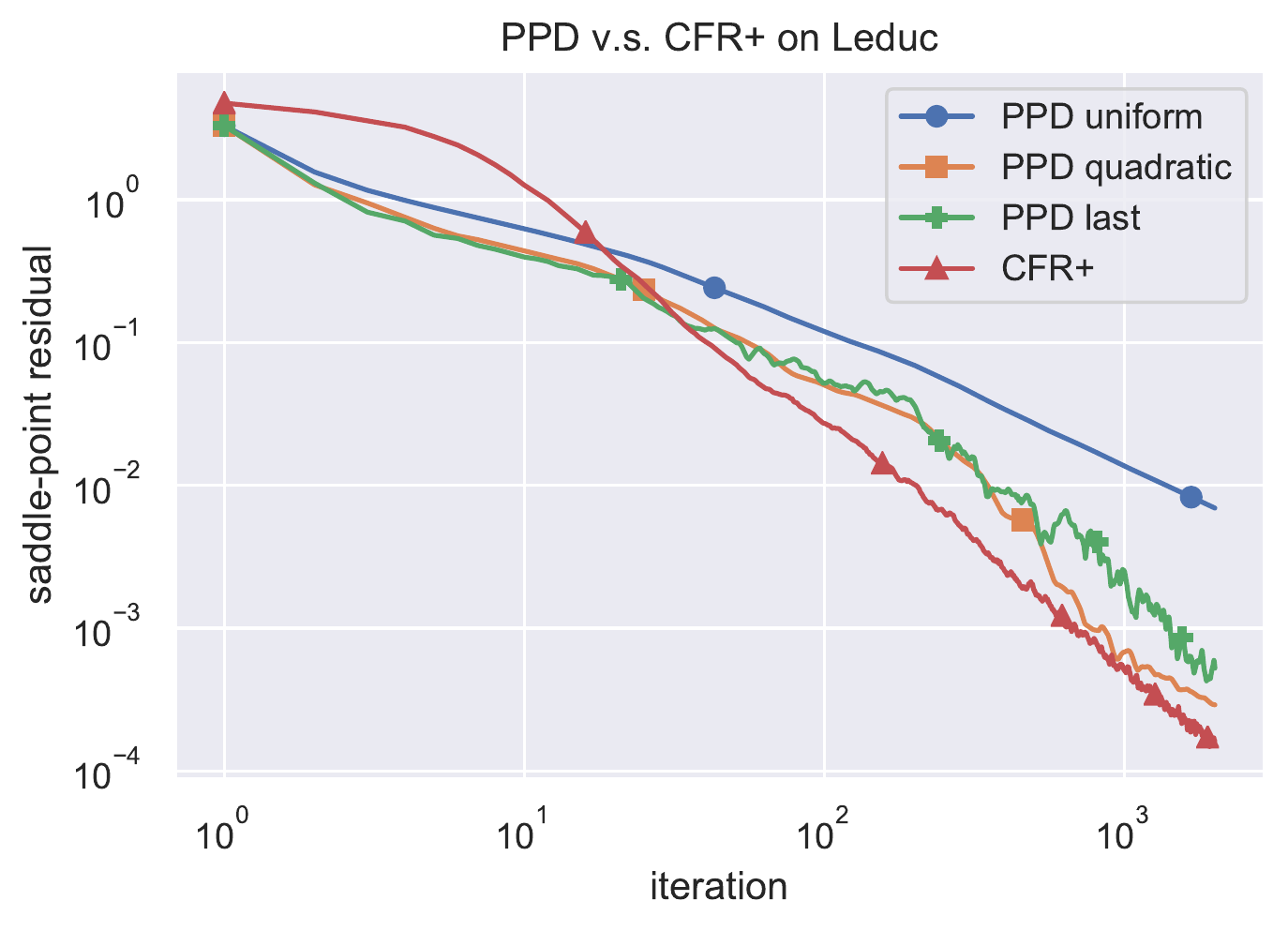} \includegraphics[width=0.48\columnwidth]{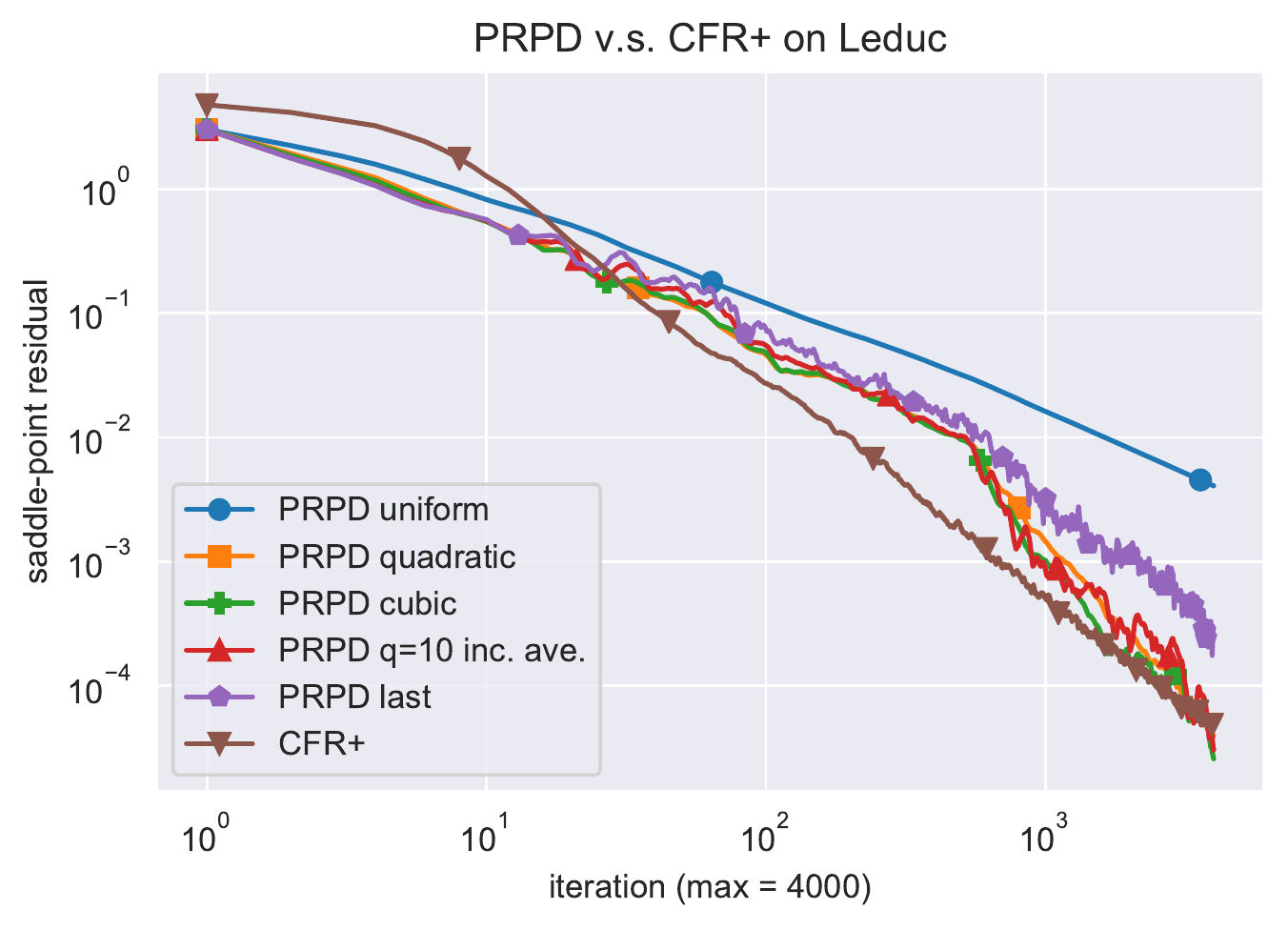}
			\caption{PPD and PRPD for EFG}
			\label{fig:precond-pda-rpda-cfr}
		\end{center}
	\end{figure}

	\subsection{Details of the TV-$\ell_1$ minimization experiments}\label{app:reconstr-images}
	Figure \ref{fig:images} displays the original, corrupted, and reconstructed images - Cameraman, Pirates and Bingley. The reconstruction is based on the quadratic averages of the PDA iterates at $T=1000$.
	
	\begin{figure}[htp]
		\begin{center}
			\minipage{0.3\textwidth}
			\includegraphics[width=\linewidth]{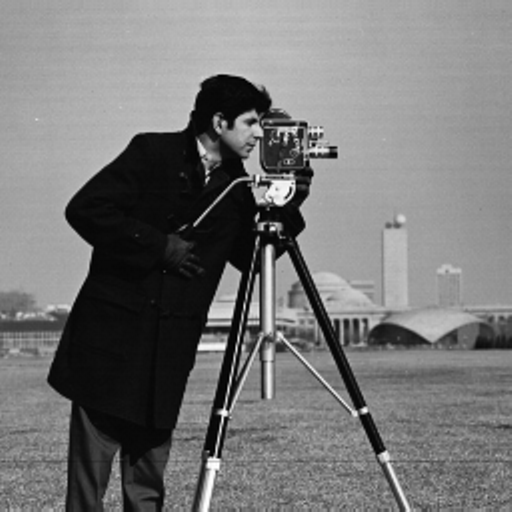} 
			\endminipage
			\minipage{0.3\textwidth}
			\includegraphics[width=\linewidth]{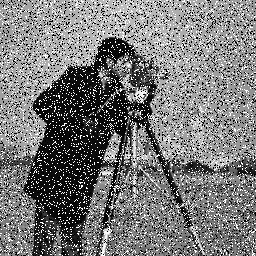} 
			\endminipage
			\minipage{0.3\textwidth}
			\includegraphics[width=\linewidth]{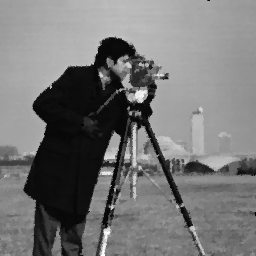} 
			\endminipage\hfill
			\minipage{0.3\textwidth}
			\includegraphics[width=\linewidth]{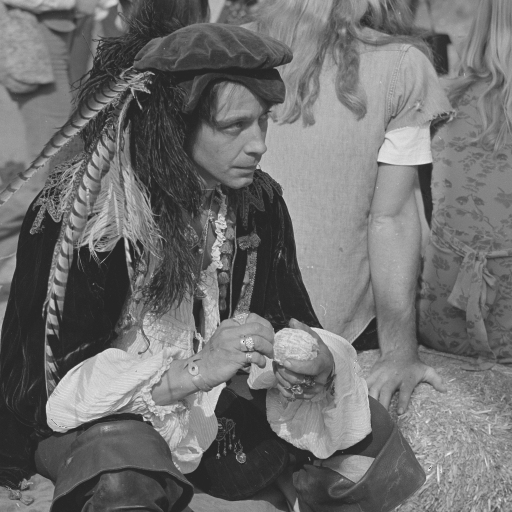} 
			\endminipage
			\minipage{0.3\textwidth}
			\includegraphics[width=\linewidth]{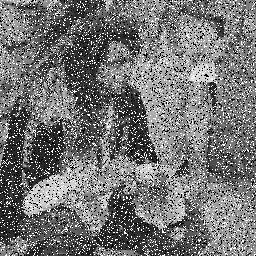} 
			\endminipage
			\minipage{0.3\textwidth}
			\includegraphics[width=\linewidth]{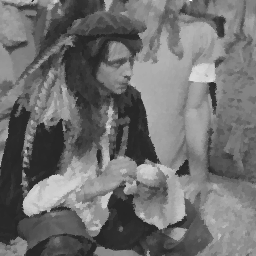} 
			\endminipage\hfill
			\minipage{0.3\textwidth}
			\includegraphics[width=\linewidth]{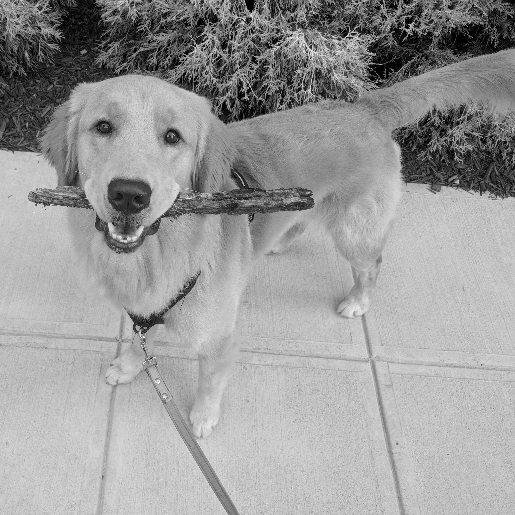} 
			\endminipage
			\minipage{0.3\textwidth}
			\includegraphics[width=\linewidth]{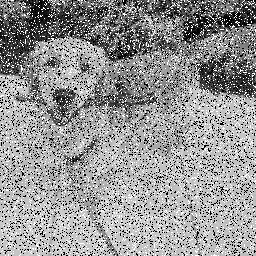} 
			\endminipage
			\minipage{0.3\textwidth}
			\includegraphics[width=\linewidth]{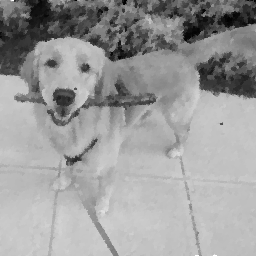} 
			\endminipage
			\caption{Original (left), corrupted (middle) and reconstructed images (right) of Cameraman (top), Pirate (middle), and Bingley (bottom)}
			\label{fig:images}
		\end{center}
		\vskip -0.5in
	\end{figure}
	
\end{document}